\documentclass{article}

\usepackage{graphicx}%
\usepackage{multirow}%
\usepackage{amsmath,amssymb,amsfonts}%
\usepackage{amsthm}%
\usepackage{mathrsfs}%
\usepackage[title]{appendix}%
\usepackage{xcolor}%
\usepackage{textcomp}%
\usepackage{manyfoot}%
\usepackage{booktabs}%
\usepackage{algorithm}%
\usepackage{algorithmicx}%
\usepackage{algpseudocode}%
\usepackage{listings}%
\usepackage{siunitx}
\usepackage{bbm}
\usepackage{authblk}
\usepackage{hyperref}
\usepackage[a4paper, total={6in, 8in}]{geometry}

\newtheorem{theorem}{Theorem}%
\newtheorem{proposition}[theorem]{Proposition}%
\newtheorem{remark}{Remark}%
\newtheorem{corollary}{Corollary}

\newcommand{\R}{\mathbb{R}}
\newcommand{\x}{\mathbf{x}}
\newcommand{\V}{\mathbf{V}}
\newcommand{\basisvec}{\mathbf{v}}
\newcommand{\opc}{\widehat{\mathbf{c}}}
\newcommand{\opA}{\widehat{\mathbf{A}}}
\newcommand{\opH}{\widehat{\mathbf{H}}}

\newcommand{\Datamatrix}{\mathbf{D}}
\newcommand{\Rhsmatrix}{\mathbf{R}}

\newcommand{\tr}{^{\top}}

\usepackage{acro}

\DeclareAcronym{dof}{
short = dof,
long = degrees of freedom,
tag = abbrev
}

\DeclareAcronym{eim}{
short = EIM,
long = empirical interpolation method,
tag = abbrev
}

\DeclareAcronym{fe}{
short = FE,
long = finite element,
tag = abbrev,
long-plural = s,
short-plural = s
}

\DeclareAcronym{fom}{
short = FOM,
long = full-order model,
tag = abbrev
}

\DeclareAcronym{geim}{
short = GEIM,
long = generalized empirical interpolation method,
tag = abbrev
}

\DeclareAcronym{ivp}{
short = IVP,
long = initial value problem,
tag = abbrev
}

\DeclareAcronym{mor}{
short = MOR,
long = model order reduction,
tag = abbrev
}

\DeclareAcronym{ode}{
short = ODE,
long = ordinary differential equation,
tag = abbrev
}

\DeclareAcronym{oed}{
short = OED,
long = optimal experimental design,
tag = abbrev
}

\DeclareAcronym{omp}{
short = OMP,
long = orthogonal matching pursuit,
tag = abbrev
}

\DeclareAcronym{opinf}{
short = OpInf,
long = operator inference,
tag = abbrev
}

\DeclareAcronym{pbdw}{
short = PBDW,
long = parameterized-background data-weak,
tag = abbrev
}

\DeclareAcronym{pde}{
short = PDE,
long = partial differential equation,
tag = abbrev,
long-plural = s,
short-plural = s
}

\DeclareAcronym{pg}{
short = PG,
long = Petrov Galerkin,
tag = abbrev
}

\DeclareAcronym{pod}{
short = POD,
long = proper orthogonal decomposition,
tag = abbrev
}

\DeclareAcronym{qoi}{
short = QoI,
long = quantity of interest,
tag = abbrev
}

\DeclareAcronym{rb}{
short = RB,
long = reduced basis,
tag = abbrev,
long-plural-form = reduced bases,
short-plural = ,
}

\DeclareAcronym{rom}{
short = ROM,
long = reduced-order model,
tag = abbrev
}

\DeclareAcronym{spd}{
short = s.p.d.,
long = symmetric positive definite,
tag = abbrev
}

\DeclareAcronym{svd}{
short = SVD,
long = singular value decomposition,
tag = abbrev
}

\DeclareAcronym{uq}{
short = UQ,
long = uncertainty quantification,
tag = abbrev
}

\DeclareAcronym{wlog}{
short = w.l.o.g.,
long = without loss of generality,
tag = abbrev
}

\begin{document}


\title{Nested Operator Inference for Adaptive Data-Driven Learning of Reduced-order Models}

\author[1]{Nicole Aretz}
\author[1]{Karen Willcox}

\affil[1]{Oden Institute for Computational Engineering and Sciences, University of Texas
at Austin, Austin, 78712, TX, United States}

\maketitle


\abstract{
This paper presents a data-driven, nested Operator Inference (OpInf) approach for learning physics-informed reduced-order models (ROMs) from snapshot data of high-dimensional dynamical systems.
The approach exploits the inherent hierarchy within the reduced space to iteratively construct initial guesses for the OpInf learning problem that prioritize the interactions of the dominant modes. 
The initial guess computed for any target reduced dimension corresponds to a ROM with provably smaller or equal snapshot reconstruction error than with standard OpInf.
Moreover, our nested OpInf algorithm can be warm-started from previously learned models, enabling versatile application scenarios involving dynamic basis and model form updates.
We demonstrate the performance of our algorithm on a cubic heat conduction problem, with nested OpInf achieving a four times smaller error than standard OpInf at a comparable offline time.
Further, we apply nested OpInf to a large-scale, parameterized model of the Greenland ice sheet where, despite model form approximation errors, it learns a ROM with, on average, \SI{3}{\%} error and computational speed-up factor above 19,000.
}


\section{Introduction}\label{sec:introduction}

We present a nested \ac{opinf} approach for learning and updating \acp{rom} of high-dimensional dynamical systems non-intrusively from snapshot data and structural knowledge of the governing equations.
By exploiting the intrinsic low-dimensionality of the system's solution manifold, \acp{rom} often achieve computational speed-ups of several orders of magnitude with limited loss in accuracy compared to the \ac{fom}.
Consequently, \acp{rom} enable real-time and long-term predictions as well as many-query applications such as optimization and forward and inverse uncertainty quantification.
We refer to the surveys \cite{benner2015survey, hesthaven2016certified, benner2017model} for further introductions into model reduction.
When the full-order operators cannot be accessed --- for example when working with commercial or legacy codes --- 
the classic construction of a \ac{rom} through Galerkin projection is not possible.
In this case, non-intrusive learning approaches are used instead to build fast surrogate models from available model information and data.
Examples include \ac{opinf} for learning projections onto linear (\cite{ghattas2021acta, kramer2024survey}) and non-linear solution manifolds (\cite{geelen2024learning, peherstorfer2022breaking}), dynamic mode decomposition (\cite{schmid2022dynamic}), sparse identification of nonlinear dynamics (\cite{brunton2016discovering}), neural operators (\cite{azizzadenesheli2024neural}), or general black-box machine learning approaches when little physical knowledge is available. 
In this paper, we focus on the \ac{opinf} method for learning Galerkin projections onto linear solution manifolds.

The \ac{opinf} method \cite{Peherstorfer2016c} learns a \ac{rom} non-intrusively from snapshot data by mimicking the structure of a Galerkin projection onto the governing equations. 
Under certain conditions, the \ac{opinf} learning problem recovers the reduced-order operators of the \ac{fom}'s intrusive Galerkin projection onto the chosen reduced space (\cite{Peherstorfer2016c,Peherstorfer2020a, aretz2024enforcing}, \cite{mcquarrie2023data}, Sec. 2.3.4).
These conditions can be met through the (computationally expensive) reprojection method (\cite{Peherstorfer2020a, Uy2021});
otherwise, the \ac{opinf} learning problem is typically embedded within an outer optimization of regularization hyper-parameters to encourage suitable stability radii (\cite{mcquarrie2021data, Sawant2021, gkimisis2025spatiallylocal}).
When specific model properties are known, these can be imposed on the learning problem; examples include Hamiltonian (\cite{sharma2022hamiltonian, gruber2023hamiltonian, geng2025data, gruber2025variational}) and Lagrangian (\cite{sharma2024lagrangian, sharma2024preserving}) structures, and preservation of energy (\cite{koike2024energy}), differential forms (\cite{filanova2023mechanical, geng2024gradient}) or equilibrium points (\cite{goyal2023guaranteed}).
However, these approaches are limited to specific classes of \acp{fom}.

With nested \ac{opinf}, we target a model property that is shared among all projection-based \acp{rom} with polynomial operators: The ability to obtain the projection of the \ac{rom} onto a smaller reduced space by taking submatrices within its matrix representation.
This property is important because, by construction, reduced spaces have an inherent hierarchy that determines the importance of each operator entry for the accuracy of the learned \ac{rom}.
Standard \ac{opinf} ignores the basis hierarchy by learning all operator entries together in a large learning problem that becomes poorly conditioned for large reduced spaces or high-order polynomial operators.
In contrast, in our nested \ac{opinf} approach \cite{aretz2024enforcing}, we exploit the basis hierarchy to obtain an optimal order for learning the reduced-order operator entries through \ac{opinf} least squares problems that are as small as possible.
However, \cite{aretz2024enforcing} relies on the reprojection method to negate exponential error propagation.
With this paper, we generalize nested \ac{opinf} beyond reprojection to keep the offline learning cost (after the selection of the reduced space) independent of the full-order dimension and robust against error propagation.

Our nested \ac{opinf} algorithm iteratively expands the reduced space and the learned reduced-order operators until the target dimension is reached.
Operators learned in previous iterations are imposed as initial guesses within the expanded learning problems as a form of nested, self-informed regularization.
We prove that this strategy yields a better initial guess for the \ac{rom} at the target reduced dimension than the Tikhonov regularization prevalent in practice, with equality only in a single, worst-case scenario.
The computational cost of the approach remains proportional to standard \ac{opinf} (with the reduced dimension as proportionality coefficient), and is, in particular, independent from the full-order dimension.
Moreover, we show that the numerical stability of the encountered least squares problems is best when learning the interactions of the most important modes in the basis hierarchy, and that --- when learning the interactions of the less important modes --- the nested algorithm is increasingly able to apply stronger regularization to the less stable least squares problems without sacrificing reconstruction accuracy.
We first present the algorithm in a minimal form that makes it easiest to implement, and then provide extensions that can further improve its performance in practice.
We illustrate the benefits of nested \ac{opinf} compared to standard \ac{opinf} on a cubic heat conduction example;
further, we show its robustness to model approximations in a large-scale model of the Greenland ice sheet.

This paper is organized as follows:
In Section \ref{sec:background}, we explain hierarchies within reduced spaces and how they are reflected in \ac{opinf} learning problems.
In Section \ref{sec:nested}, we introduce our nested \ac{opinf} algorithm and analyze it in comparison to standard \ac{opinf}.
We discuss possible extensions of the algorithm in Section \ref{sec:extensions}, and demonstrate it in Section \ref{sec:results} on two numerical examples.
Finally, we conclude in Section \ref{sec:conclusion}.


\section{Hierarchical reduced-order modelling}\label{sec:background}

Our goal is to build a \ac{rom} of a dynamical system from structural knowledge of the \ac{fom}'s governing equations
\begin{equation}\label{eq:fom}
    \begin{aligned}
    \dot{\x}(t) &= \mathbf{c} + \mathbf{A} \x(t) + \mathbf{H}[\x(t) \otimes \x(t)], \quad t > 0,\\
    \x(0) &= \x_0,
    \end{aligned}
\end{equation}
and snapshot data $\x(t_1), \dots, \x(t_K)$ of its solution $\x : [0, T] \rightarrow \R^n$ at time steps $0 = t_1 < t_2 < \dots < t_K \le T$.
The vector $\mathbf{c} \in \R^n$ encodes a time-invariant forcing function;
the matrix $\mathbf{A} \in \R^{n\times n}$ encodes the actions of linear operators in the original \ac{pde}; the matrix $\mathbf{H} \in \R^{n \times n^2}$ encodes operations on quadratic interactions of the state $\x$ with itself.
Note that the restriction to a time-invariant quadratic system of the form (\ref{eq:fom}) is only to simplify the exposition; our concepts extend to time-varying forcing functions, non-Euclidean inner products, and to parameterized, higher-order polynomial or non-polynomial operators through the techniques introduced in \cite{Peherstorfer2016c,qian2022reduced, McQuarrie2021c, aretz2024enforcing, Qian2020, benner2020operator}.

To facilitate projection-based model reduction, we use \ac{pod} on the snapshots $\x(t_1), \dots, \x(t_K)$ to identify an ordered sequence of orthonormal vectors $\basisvec_1, \dots, \basisvec_K \in \R^n$ such that, for all $1 \le s \le K$, the squared projection error of the snapshots onto $
\text{span}\{\basisvec_1, \dots, \basisvec_s\} \subset \R^n$ is the minimum among all $s$-dimensional subspaces of $\R^n$, i.e.,
\begin{align}\label{eq:POD:minimization}
    \sum_{k=1}^K \| \x(t_k) - \V_s \V_s\tr \x(t_k) \|^2 = 
    \min_{\substack{\mathbf{W} \in \mathbb{R}^{n \times s},\\ \mathbf{W}\tr \mathbf{W} = \mathbf{I}_s}} ~ \sum_{k=1}^K\| \x(t_k) - \mathbf{W} \mathbf{W}\tr \x(t_k) \|^2
    =: \varepsilon_s^2
\end{align}
holds for all $1 \le s \le K$ with
\begin{align}
    \V_s := [\basisvec_1, \dots, \basisvec_s] \in \mathbb{R}^{n \times s}, \quad 
    \V_s^{\top} \V_s = \mathbf{I}_s
\end{align}
and $\mathbf{I}_s$ the identity matrix of dimension $s \times s$.

We use the sequence $\varepsilon_1^2 \ge \varepsilon_2^2 \ge \dots \ge \varepsilon_K^2 = 0$ of residual errors to choose a (target) reduced dimension $r \le K$ for our \ac{rom}.
With $r$ fixed, our goal becomes to identify reduced-order operators $\widehat{\mathbf{c}}_r \in \R^r$, $\widehat{\mathbf{A}}_r \in \R^{r\times r}$, $\widehat{\mathbf{H}}_r \in \R^{r \times r^{(2)}}$ with $r^{(2)} := \frac{1}{2} r (r+1)$ such that $\x(t) \approx \V_r \widehat{\x}_r(t)$ for the solution $\widehat{\x}_r : [0, T] \rightarrow \R^r$ to the dynamical system
\begin{equation}\label{eq:rom}
    \begin{aligned}
    \dot{\widehat{\x}}_r(t) 
    &= \widehat{\mathbf{c}}_r + \widehat{\mathbf{A}}_r\widehat{\x}_r(t) + \widehat{\mathbf{H}}_r[\widehat{\x}_r(t) \, \widehat{\otimes} \ \widehat{\x}_r(t)], \quad t > 0 \\
    \widehat{\x}_r(0) &= \V_r\tr \x_0.
\end{aligned}
\end{equation}
Here, we are using the condensed Kronecker product $\widehat{\otimes}$ defined through
\begin{align}\label{eq:condensed-Kronecker}
    \mathbf{w} \widehat{\otimes} \mathbf{w} := [w_1^2, w_1w_2, w_2^2, w_1 w_3, \dots, w_r^2]\tr \in \R^{r^{(2)}}
\end{align}
for all $\mathbf{w} = [w_1, \dots, w_r]\tr \in \R^r$ (and analogously for other dimensions).
The reduced dimension $r$ is typically small (order of tens) compared to the full-order dimension $n$ (order of tens of thousands to millions), such that even with dense reduced-order operators $\widehat{\mathbf{c}}_r$, $\widehat{\mathbf{A}}_r$, $\widehat{\mathbf{H}}_r$,  Eq.~\eqref{eq:rom} can be solved fast, often with several orders of magnitude in speed-up.

With the \ac{opinf} method, the reduced-order operators $\opc_r, \opA_r, \opH_r$ are learned through the least squares problem
\begin{align}\label{eq:opinf:nonregularized}
    \min_{\widehat{\mathbf{O}}_r = [\widehat{\mathbf{c}}_r, \widehat{\mathbf{A}}_r, \widehat{\mathbf{H}}_r]}
    \| \Datamatrix_r \widehat{\mathbf{O}}_r\tr - \Rhsmatrix_r \|_F^2
\end{align}
with data matrix $\Datamatrix_r \in \R_r^{K \times r_{\rm{tot}}}$, $r_{\rm{tot}} := 1 + r + r^{(2)}$, and time-derivative matrix $\Rhsmatrix_r \in \R^{K \times r}$ defined as
\begin{align}\label{eq:opinf:matrices}
    \Datamatrix_r &:= \left[
    \begin{array}{ccc}
        ~1,~ & ~\mathbf{p}_r(t_1)\tr,~ & ~[\mathbf{p}_r(t_1) \widehat{\otimes}\mathbf{p}_r(t_1)]\tr~ \\
        \vdots & \vdots & \vdots~~ \\
        1, & \mathbf{p}_r(t_K)\tr, & ~[\mathbf{p}_r(t_K) \widehat{\otimes}\mathbf{p}_r(t_1)]\tr \\
    \end{array}
    \right],
    &\Rhsmatrix_r &:= \left[ 
    \begin{array}{c}
        \dot{\mathbf{p}}_r(t_1)\tr \\
        \vdots \\
        \dot{\mathbf{p}}_r(t_K)\tr
    \end{array}
    \right]
\end{align}
using the projection $\mathbf{p}_r(t_k) := \V_r\tr \x(t_k)$ of each snapshot $\x(t_k)$, $1\le k \le K$.
The derivative $\dot{\mathbf{p}}_r = \V_r\tr \dot{\x}$ is typically approximated using finite differences, albeit other techniques are also possible (c.f., \cite{Peherstorfer2020a, Uy2021}).

The motivating advantage for using the \ac{opinf} method is that it is non-intrusive: 
The reduced-order operators $\opc_r, \opA_r, \opH_r$ are obtained without accessing the full-order operators $\mathbf{c}, \mathbf{A}, \mathbf{H}$.
This, in turn, enables reduced-order modelling for commercial or legacy codes.
The downside is, in turn, that the quality of the learned \ac{rom} depends heavily on the amount of available training data:
The minimization \eqref{eq:opinf:nonregularized} has a unique solution if and only if $\Datamatrix_r$ has full column rank, which requires $K \ge r_{\rm{tot}}$ snapshots.
Even then, any perturbations in the time-derivative matrix $\Rhsmatrix$ are scaled by the inverse $\sigma_{\min}(\Datamatrix_r)^{-1}$ of the smallest singular value of $\Datamatrix_r$, increases in $r$ (c.f., \cite{aretz2024enforcing} or Corollary \ref{thm:opinf:conditioning} in Appendix \ref{sec:proofs}).
Therefore, choosing a larger reduced dimension $r$ for the purpose of improving how well $\x$ can possibly be approximated in $\V_r$ destabilizes the \ac{opinf} learning problem.
Even when regularization is applied for stability, the \ac{rom} learned for the target reduced dimension $r$ may indeed perform worse than one learned for a smaller dimension $s < r$.

Learning a reduced-order approximation $\x(t) \approx \V_s \widehat{\x}_s(t)$ where $\widehat{\x}_s : [0, T] \rightarrow \R^s$ solves
\begin{equation}\label{eq:rom:s}
    \begin{aligned}
    \dot{\widehat{\x}}_s(t) 
    &= \widehat{\mathbf{c}}_s + \widehat{\mathbf{A}}_s\widehat{\x}_s(t) + \widehat{\mathbf{H}}_s[\widehat{\x}_s(t) \widehat{\otimes} \widehat{\x}_s(t)], \quad t > 0 \\
    \widehat{\x}_s(0) &= \V_s \tr \x_0
\end{aligned}
\end{equation}
with a smaller reduced dimension $s < r$ has the advantage that the data matrix $\Datamatrix_s$ (defined analogously to \eqref{eq:opinf:matrices}) is better conditioned than $\Datamatrix_r$.
Moreover, because the sequence $\basisvec_1, \dots, \basisvec_s, \dots, \basisvec_r$ were chosen with \ac{pod}, the reduced basis $\V_s$ neglects exactly those basis vectors $\basisvec_{s+1}, \dots, \basisvec_r$ in $\V_r$ that are the least important for the snapshot approximation accuracy in \eqref{eq:POD:minimization}:
While $\varepsilon_s^2 \le \varepsilon_{s+1}^2 \le \varepsilon_{s+2}^2$ for any $s \ge 1$, we also have that $\varepsilon_{s+1}^2 - \varepsilon_s^2 \ge \varepsilon_{s+2}^2 - \varepsilon_{s+1}^2 \ge 0$. That is, 
while the total approximation accuracy of the snapshots improves when the reduced space $\V_s$ is expanded with the next \ac{pod} basis vectors $\basisvec_{s+1}, \basisvec_{s+2}$, the incremental accuracy gain diminishes.
Improving the recovery of trajectories $\mathbf{p}_s(t) = \V_s\tr \x(t)$ for the smallest reduced dimensions $s \ll r$ associated to the most important modes $\V_s$ can therefore have a stronger effect on the approximation accuracy of a learned \ac{rom} than increasing the reduced dimension itself.

To fully exploit the basis hierarchy while still learning a \ac{rom} for the target dimension $r$, we propose to iteratively learn operator updates $\Delta^c_{s} \in \R^{s}, \Delta^A_{s} \in \R^{s \times s}, \Delta^H_{s} \in \R^{s \times s^{(2)}}$ to iteratively expand
\begin{equation}
\begin{aligned}\label{eq:updates:s}
    \widehat{\mathbf{c}}_{s} &= \left(
    \begin{array}{c}
        \widehat{\mathbf{c}}_{s-1} \\
        0
    \end{array}\right) + \Delta^c_{s} =: \opc_{s}^{(0)} + \Delta^c_{s}, \\
    \widehat{\mathbf{A}}_{s} &= \left(
    \begin{array}{cc}
        \widehat{\mathbf{A}}_{s-1} & \mathbf{0}_{{s-1} \times 1}\\
        \mathbf{0}_{1 \times {s-1}} & 0
    \end{array}\right) + \Delta^A_{s} =: \opA_{s}^{(0)} + \Delta^A_{s}, \\
    \widehat{\mathbf{H}}_{s} &= \left(
    \begin{array}{cc}
        \widehat{\mathbf{H}}_{s-1} & \mathbf{0}_{{s-1} \times {s-1}}\\
        \mathbf{0}_{1 \times ({s-1})^{(2)}} & \mathbf{0}_{1 \times {s-1}}
    \end{array}\right) + \Delta^H_{s} =: \opH_{s}^{(0)} + \Delta^H_{s}
\end{aligned}
\end{equation}
from $s=2$ until $s=r$, starting from operators $\opc_1, \opA_1, \opH_1$ learned for $\V_1$.
In \eqref{eq:updates:s}, we have defined the reduced-order operators $\opc_{s}^{(0)} \in \R^s$, $\opA_{s}^{(0)} \in \R^{s \times s}$, $\opH_{s}^{(0)} \in \R^{s \times s}$ as the trivial extension of the operators $\opc_{s-1}, \opA_{s-1}, \opH_{s-1}$ from the smaller reduced space $\V_{s-1}$ into $\V_s = [\V_{s-1}, \basisvec_s]$ such that any entry associated to the additional basis vector $\basisvec_s$ is set to zero.
We use $\opc_s^{(0)}, \opA_s^{(0)}, \opH_s^{(0)}$ as initial guesses within our nested \ac{opinf} algorithm, and their corresponding \ac{rom} solution $\widehat{\x}_s^{(0)} : [0, T] \rightarrow \R^s$ with
\begin{equation}\label{eq:rom:s0}
    \begin{aligned}
    \dot{\widehat{\x}}_s^{(0)}(t) 
    &= \widehat{\mathbf{c}}_s^{(0)} + \widehat{\mathbf{A}}_s^{(0)}\widehat{\x}_s^{(0)}(t) + \widehat{\mathbf{H}}_s^{(0)}[\widehat{\x}_s^{(0)}(t) \widehat{\otimes} \widehat{\x}_s^{(0)}(t)], \quad t > 0 \\
    \widehat{\x}_s^{(0)}(0) &= \V_s \tr \x_0
\end{aligned}
\end{equation}
as a reference trajectory to guarantee that any learned updates $\Delta_s^c, \Delta_s^A, \Delta_s^H$ indeed improve upon the \ac{rom} learned for $\V_{s-1}$.
Through our nested expansion procedure from $\opc_{2} = \opc_{1}^{(0)} + \Delta_2^c$ to $\opc_{r} = \opc_{r-1}^{(0)} + \Delta_{r}^c$ (and simultaneously for $\opA_r$ and $\opH_r$),
the entries associated to the most important basis vectors $\basisvec_1, \basisvec_2, \dots$ are learned first and in stable and small learning problems.
These entries are then used to inform the later updates $\Delta_s^c, \Delta_s^H, \Delta_s^H$, $1 \ll s \le r$, for the less important modes.
The final \ac{rom} \eqref{eq:rom} learned for the target dimension $r$ is then imbued with all \acp{rom} on the subspaces $\V_s$, $s=1, \dots, r-1$.


\section{Nested Operator Inference}\label{sec:nested}

In this section, we present a general, nested \ac{opinf} learning algorithm that can be implemented easily, especially if starting from an existing \ac{opinf} code.
We discuss computational cost and analyze the change in \ac{rom} accuracy over the course of the nested expansion process.

\subsection{Algorithm Description}\label{sec:algorithm}

\begin{algorithm}
\caption{Nested Operator Inference}\label{alg:flower}
\begin{algorithmic}[1]
\smallskip
\State \textbf{Input:}
reduced dimension $r \in \mathbb{N}$ of (target) reduced space $\V_r \in \R^{n \times r}$, 
$K \in \mathbb{N}$ projected snapshots matrix $\mathbf{P}_r = [\V_r\tr \x(t_1), \dots, \V_r\tr \x(t_K)]\tr \in \R^{K \times r}$,
$n_{\omega} \in \mathbb{N}$ combinations $\{\omega^{(i)} = (\omega_{c}^{(i)}, \omega_{A}^{(i)}, \omega_{H}^{(i)})\}_{i=1}^{n_{\omega}}$ of positive training weights 
\State \textbf{Optional inputs:}
start dimension $1 \le r_{0} \le r$ (default $1$),
initial guesses $\opc_{r_{0}}^{(0)} \in \R^{r_{0}}$, $\opA_{r_{0}}^{(0)} \in \R^{r_{0} \times r_{0}}$, $\opH_{r_{0}}^{(0)} \in \R^{r_{0} \times r_{0}^{(2)}}$ (default $\mathbf{0}$ in given dimension), relative error threshold $\bar{\delta} \ge 0$ (default 0)
\State \textbf{Output:} Reduced operators $\opc_{r} \in \R^{r}$,
$\opA_{r} \in \R^{r \times r}$,
$\opH_{r} \in \R^{r \times r^{(2)}}$
\Statex \hrulefill
\smallskip
\Statex \texttt{\# Preparation}
\State Use finite differences to approximate $\mathbf{R} \gets \dot{\mathbf{P}} \in \R^{K \times r}$
\smallskip
\Statex \texttt{\# Iterate over reduced dimension}
\For{$s = r_0, \dots, r$}
\smallskip
\Statex \quad \texttt{\# Setup \ac{opinf} matrices}
\State $\mathbf{P}_s \gets [\text{col}_1(\mathbf{P}_r), \dots, \text{col}_s(\mathbf{P}_r)] \in \R^{K \times s}$ (restriction onto first $s$ columns)
\State $\Datamatrix_{s} \gets [\mathbf{1}_{K \times 1}, \mathbf{P}_s, [\mathbf{P}_s \odot \mathbf{P}_s]\tr \mathbf{C}_s] \in \R^{K \times s_{\rm{tot}}}$ (data matrix for reduced space $\V_s$)
\State $\Rhsmatrix_s \gets [\text{col}_1(\Rhsmatrix), \dots, \text{col}_s(\Rhsmatrix)] \in \R^{K \times s}$ (restriction onto first $s$ columns)
\smallskip
\Statex \quad \texttt{\# Expand operators as new initial guesses}
\If{$s \neq r_0$}
\State $\opc_{s}^{(0)} \gets \left[\begin{array}{c}
    \opc_{s-1} \\
    0
\end{array}\right] \in \R^s$, 
$\opA_{s}^{(0)} \gets \left[\begin{array}{cc}
    \opA_{s-1} & \mathbf{0} \\
    \mathbf{0} & 0
\end{array}\right] \in \R^{s \times s}$ \label{alg:flower:def:c0A0}
\State
$\opH_{s}^{(0)} \gets \left[\begin{array}{cc}
    \opH_{s-1} & \mathbf{0} \\
    \mathbf{0} & \mathbf{0}
\end{array}\right] \in \R^{s \times s^{(2)}}$\label{alg:flower:def:H0}
\EndIf
\smallskip
\Statex \quad \texttt{\# Evaluate reference error}
\State Solve $\dot{\widehat{\x}} = \opc_{s}^{(0)}+ \opA_{s}^{(0)}\widehat{\x} +\opH_{s}^{(0)} [\widehat{\x} \widehat{\otimes} \widehat{\x}]$ ($0 < t \le t_K$), $\widehat{\x}(0) = \text{row}_1(\mathbf{P}_s)\tr$
\State $\delta_s^{(0)} \gets \|\mathbf{P}_s\tr - [\widehat{\x}(t_1), \dots, \widehat{\x}(t_K)]\|_F^2$, $\sigma_s^{(0)} \gets \infty$
\smallskip
\Statex \quad \texttt{\# Evaluate ROM error for different regularization weights}
\For{$i = 1, \dots, n_{\omega}$}
\smallskip
\State $\mathbf{w} \gets [\omega_{c}^{(i)}, \omega_{A}^{(i)} \mathbf{1}_{1 \times s}, \omega_{H}^{(i)} \mathbf{1}_{1 \times s^{(2)}}] \in \R^{s \times s_{\rm{tot}}}$ (weight vector) \label{alg:flower:def:w}
\State $\opc_{s}^{(i)}, \opA_{s}^{(i)}, \opH_{s}^{(i)}, \sigma_s^{(i)} \gets \texttt{OpInf}(s, \Datamatrix_s, \Rhsmatrix_s, \mathbf{w}, \opc_{s}^{(0)}, \opA_{s}^{(0)}, \opH_{s}^{(0)})$ \label{alg:flower:def:learn}
\State Solve $\dot{\widehat{\x}} = \opc_{s}^{(i)} + \opA_{s}^{(i)}\widehat{\x} + \opH_{s}^{(i)} [\widehat{\x} \widehat{\otimes} \widehat{\x}] $ ($0 < t \le t_K$), $\widehat{\x}(0) = \text{row}_1(\mathbf{P}_s)\tr$
\State Compute error $\delta_s^{(i)} \gets \|\mathbf{P}_s\tr - [\widehat{\x}(t_1), \dots, \widehat{\x}(t_K)]\|_F^2$
\smallskip
\EndFor
\smallskip
\Statex \quad \texttt{\# Choose regularization}
\State $i^* \gets \text{arg} \max_{0 \le i \le n_{\omega}} \sigma_s^{(i)}$ s.t. $\delta_i \le (1 + \bar{\delta}) \min_{0 \le j \le n_{\omega}} \delta_j$ \label{alg:flower:def:istar}
\State $\opc_{s} \gets \opc_{s}^{(i^*)}, \opA_{s} \gets \opA_{s}^{(i^*)}, \opH_{s} \gets \opH_{s}^{(i^*)}$, $\delta_s^* \gets \delta_s^{(i^*)}$, $\sigma_s^* \gets \sigma_s^{(i^*)}$
\smallskip
\EndFor \\
\smallskip
\Return $\opc_{r}, \opA_{r}, \opH_{r}$
\smallskip
\end{algorithmic}
\end{algorithm}

Algorithm \ref{alg:flower} is a nested \ac{opinf} learning procedure that iteratively learns 
\acp{rom} \eqref{eq:rom:s} for the reduced space $\V_s$
starting from $s = 1$, then $s=2$, until $s=r$ reaches the target dimension $r$.
In iteration $s$, the operators $\opc_s, \opA_s, \opH_s$ are learned by solving the regularized \ac{opinf} learning problem
\begin{equation}\label{eq:opinf:regularized:initialguess}
    \begin{aligned}
    \min_{\opc_s, \opA_s, \opH_s} &\|\Datamatrix_s \left[\begin{array}{c}
        \opc_s\tr\\ 
        \opA_s\tr\\ 
        \opH_s\tr
    \end{array}\right] - \Rhsmatrix_s \|_F^2 \\
    &+ \omega_c^2\| \opc_s-\opc_s^{(0)}\|_F^2 + \omega_A^2\|\opA_s-\opA_s^{(0)}\|_F^2 + \omega_H^2\|\opH_s-\opH_s^{(0)}\|_F^2.
\end{aligned}
\end{equation}
The regularization penalizes deviations from $\opc_{s-1}, \opA_{s-1}, \opc_{s-1}$ for all entries that encode interactions within the reduced space $\V_{s-1}$.
This is equivalent to learning updates $\Delta^c_{s}, \Delta^A_{s}, \Delta^H_{s}$ to $\opc_s^{(0)}, \opA_s^{(0)}, \opH_s^{(0)}$ in the form of \eqref{eq:updates:s} that are regularized towards zero.
Within Algorithm \ref{alg:flower}, the learning problem \eqref{eq:opinf:regularized:initialguess} is solved in the call \texttt{OpInf}; a reference implementation is provided in Appendix \ref{sec:referencecode}.

The regularization weights $\omega_c, \omega_A, \omega_H \ge 0$ in \eqref{eq:opinf:regularized:initialguess} are 
selected anew in each iteration.
To this end, Algorithm \ref{alg:flower} loops over a set of $n_{\omega} \in \mathbb{N}$ candidate combinations $\{(\omega_c^{(i)}, \omega_A^{(i)}, \omega_H^{(i)})\}_{i=1}^{n_{\omega}}$ to learn $n_{\omega}$ operators $\opc_s^{(i)}, \opA_s^{(i)}, \opH_s^{(i)}$.
For each triple of candidate operators including $(\opc_s^{(0)}, \opA_s^{(0)}, \opH_s^{(0)})$, i.e.,  for $i=0, \dots, n_{\omega}$, Algorithm \ref{alg:flower} solves the corresponding \ac{rom} \eqref{eq:rom:s} and evaluates the reconstruction error
\begin{align}
    \delta_s^{(i)} := \sum_{k=1}^K \|\mathbf{p}_s(t_k) - \widehat{\x}_s(t_k)\|_2^2 
    ~\text{ s.t. $\widehat{\x}$ solves \eqref{eq:rom:s} with operators } \opc_s^{(i)}, \opA_s^{(i)}, \opH_s^{(i)} 
\end{align}
of the projected snapshots $\mathbf{p}_s(t_k) := \V_s\tr \x(t_k)$, $1 \le k \le K$.
Under the default setting $\overline{\delta} = 0$, Algorithm \ref{alg:flower} chooses the combination $i^*$ of regularization weights $(\omega_c^{(i^*)}, \omega_A^{(i^*)}, \omega_H^{(i^*)})$ with minimal reconstruction error $\delta^{i^*}_s = \delta^{*}_s := \min_{0 \le i \le n_{\omega}} \delta^{(i)}$, using the smallest singular value $\sigma^{(i^*)}_s$ of the corresponding (regularized) data matrix as a tie breaker.
This optimization specifically includes the \ac{rom} with the reference operators $\opc_s^{(0)}, \opA_s^{(0)}, \opH_s^{(0)}$ to guarantee that the reconstruction error $\delta_s^*$ is in the worst case still close to the optimized error $\delta_{s-1}^*$ from the previous iteration (c.f. Section \ref{sec:analysis:expansionerror}).

With its optional input parameter $\overline{\delta}$, Algorithm \ref{alg:flower} permits a change to the regularization weight optimization from $i^* \in \text{arg} \min_{0 \le i \le n_{\omega}} \delta^{(i)}_s$ for $\overline{\delta}=0$ to
\begin{align}\label{eq:weights:sigma}
    i^* \in \text{arg} \min_{0 \le i \le n_{\omega}} \sigma_s^{(i)} \text{ s.t. } \delta_s^{(i)} \le (1 + \overline{\delta}) \min_{0 \le j \le n_{\omega}} \delta_s^{(j)},
\end{align}
where $\sigma_s^{(i)}$ is the smallest singular value of the regularized least squares problem \eqref{eq:opinf:regularized:initialguess} with weights $\omega_c^{(i)}, \omega_A^{(i)}, \omega_H^{(i)}$.
Rather than using the smallest singular value as a tiebreaker between equal minimal reconstruction errors as in the case $\overline{\delta} = 0$, the choice $\overline{\delta} > 0$ thus permits choosing a non-minimal $\delta_s^{i^*}$ in exchange for increased stability of the least-squares learning problem.
This avoids overfitting (see analysis in Section \ref{sec:analysis:stability}).
Moreover, because the larger singular values are induced by stronger regularization towards the reference operators $\opc_s^{(0)}, \opA_s^{(0)}, \opH_s^{(0)}$, the choice $\overline{\delta} > 0$ indirectly rewards similarity to \acp{rom} from previous iterations.

\begin{remark}\label{rmk:minerrorreduction}
    We set $\sigma_s^{(0)} = \infty$ because enforcing the initial guesses $\opc_s^{(0)}, \opA_s^{(0)}, \opH_s^{(0)}$ is equivalent to solving \eqref{eq:opinf:regularized:initialguess} with weights $\omega_c = \omega_A = \omega_H = \infty$.
    As a consequence, if $\delta_s^{(0)} \le (1 + \overline{\delta}) \min_{0 \le j \le n_{\omega}} \delta_s^{(j)}$, then the maximization over the smallest singular values guarantees that $\delta_s^* = \delta_s^{(0)}$; otherwise, 
    \begin{align}
        \delta_s^{*} &\le (1 + \overline{\delta}) \min_{0 \le j \le n_{\omega}} \delta_s^{(j)} < \delta_s^{(0)}.
    \end{align}
    Consequently, we only accept updates to the initial guesses $\opc_s^{(0)}, \opA_s^{(0)}, \opH_s^{(0)}$ if the \ac{rom}'s error is reduced by at least a factor $(1 + \overline{\delta})$, and we reject any updates that do not yield this improvement.
\end{remark}

As another set of optional inputs, Algorithm \ref{alg:flower} permits starting the nested construction of $\opc_r, \opA_r, \opH_r$ from an arbitrary dimension $r_0 \le r$ with user-defined initial guesses $\opc_{r_0}^{(0)}, \opA_{r_0}^{(0)}, \opH_{r_0}^{(0)}$.
This facilitates several use cases of the algorithm: 
Updates to hierarchical approximation of the \ac{fom} model structure (e.g., Taylor series), an iterative expansion of the reduced space from $\V_{r_0}$, $r_0 < r$, to $\V_r$ during training time, or warm-starting.
For example, warm-starting may be used in the following way: Suppose a first \ac{rom} was built using regularized \ac{opinf} for the target dimension $r$, but it fails to meet qualitative needs (e.g., it exhibits nonphysical behavior).
A second \ac{rom} is built for a smaller dimension $r_0$ to isolate whether the issue was introduced by the large dimension $r$ and consequent ill-posedness of the data matrix $\Datamatrix_r$, or if there might be another issue.
If the second \ac{rom} for dimension $r_0<r$ meets the qualitative requirements (as far as can be expected for the coarser accuracy $\varepsilon_{r_0} > \varepsilon_r$), then it can be used as an initial guess for Algorithm \ref{alg:flower} to add in the additional details of higher-order modes $\basisvec_{r_0+1}, \dots, \basisvec_{r}$, and the regularization weights can be tuned to find an appropriate balance.

\subsection{Computational Cost}

The majority of the computational cost in iteration $s$ of Algorithm \ref{alg:flower} is caused by the $n_{\omega}$ least squares solves and the $n_{\omega} + 1$ \ac{rom} evaluations.
After accounting for the regularization, a least squares solve via a singular value decomposition scales as $\mathcal{O}((K + s_{\rm{tot}})s_{\rm{tot}}^2) = \mathcal{O}(Ks^4 + s^6)$.
The cost of solving the \ac{rom} \eqref{eq:rom:s} depends on the chosen time-stepping scheme; for an explicit method it scales as $\mathcal{O}(Kss_{\rm{tot}}) = \mathcal{O}(Ks^3)$.
Iteration $s$ hence scales as $\mathcal{O}(n_{\omega}s^4(K + s^2))$.
Since $s \le r$ and there are at most $r$ iterations, the total cost of Algorithm \ref{alg:flower} is thus $\mathcal{O}(n_{\omega}r^5(K + r^2))$, independent of the \ac{fom} dimension $n \gg r$.

Compared to regularized \ac{opinf} with the same least squares solver and optimization over the same weights, the cost of Algorithm \ref{alg:flower} is larger by factor $r$.
However, because $r$ is usually small (order of tens), this additional cost remains insignificant compared to the computation of the snapshots.
Moreover, similar to the standard \ac{opinf} approach, the runtime of Algorithm \ref{alg:flower} can be improved by parallelizing the loop over the weights and using adaptive singular value decompositions to deal with different weights for the same data matrix $\Datamatrix_s$.
In our experience Algorithm \ref{alg:flower} is much less sensitive to the regularization parameters, especially in its later iterations, such that one can choose a smaller $n_{\omega}$ compared to the standard \ac{opinf} optimization or adjust the candidate weights in the course of the algorithm (smaller weights in the first iterations, larger weights in the later ones).

\subsection{Analysis of the expansion error}\label{sec:analysis:expansionerror}

The expansion with zeros of the optimized operators $\widehat{\mathbf{c}}_{s-1}, \widehat{\mathbf{A}}_{s-1}, \widehat{\mathbf{H}}_{s-1}$ from iteration $s-1$ to $\widehat{\mathbf{c}}_{s}^{(0)}, \widehat{\mathbf{A}}_{s}^{(0)}, \widehat{\mathbf{H}}_{s}^{(0)}$ as initial guesses in iteration $s$ is a central element of Algorithm \ref{alg:flower}.
The following proposition quantifies the effect this expansion has on the approximation accuracy of the corresponding \ac{rom}.

\begin{proposition}\label{thm:expansionerror}
    Let $\x : [0, T] \rightarrow \R^n$ be the solution to the \ac{fom} (\ref{eq:fom}).
    For any $1 < s \le r$, let 
    $\widehat{\mathbf{c}}_{s-1}, \widehat{\mathbf{A}}_{s-1}, \widehat{\mathbf{H}}_{s-1}$
    be the reduced operators obtained at the end of iteration $s-1$ of Algorithm \ref{alg:flower}, and let $\widehat{\x}_{s-1} : [0, T] \rightarrow \R^{s-1}$ be the corresponding solution to \eqref{eq:rom:s} with reduced dimension $s-1$.
    Let $\opc_s^{(0)} \in \R^{s}$, $\opA_s^{(0)} \in \R^{s \times s}$, $\opH_s^{(0)} \in \R^{s \times s^{(2)}}$ be the expansion of $\widehat{\mathbf{c}}_{s-1}$, $\opA_{s-1}$, $\opH_{s-1}$ as defined in \eqref{eq:updates:s} and lines \ref{alg:flower:def:c0A0}, \ref{alg:flower:def:H0} of Algorithm \ref{alg:flower}.
    Let $\widehat{\x}_{s}^{(0)} : [0, T] \rightarrow \R^{s}$ be the corresponding reduced-order solution to \eqref{eq:rom:s0}.
    Then, for all $t \in [0, T]$,
    \begin{align}\label{eq:expansionerror}
        \|\V_{s} \widehat{\x}_{s}^{(0)}(t) - \x(t)\|_2^2 = \|\V_{s-1} \widehat{\x}_{s-1}(t) - \x(t)\|_2^2 + \zeta_s(t)
    \end{align}
    with $\zeta_s(t) := (\basisvec_{s}\tr \x_0)(\basisvec_{s}\tr \x_0 - 2 \basisvec_{s}\tr \x(t))$.
\end{proposition}

\begin{proof}
We first define the auxiliary variable $\alpha(t) := \basisvec_{s}\tr \x(t) \in \R$ as the projection coefficient of the \ac{fom} solution $\x(t)$ onto the newly added basis vector $\basisvec_{s}$.
Recalling the definition $\mathbf{p}_{s-1}(t) := \V_{s-1}\tr \x(t)$, we write $\x(t)$ into the form
\begin{align}
    \x(t) = \V_{s-1} \mathbf{p}_{s-1}(t) + \basisvec_{s} \alpha(t)+ \mathbf{q}(t)
\end{align}
where $\mathbf{q}(t) := \x(t) - \V_{s-1} \mathbf{p}_{s-1}(t) - \basisvec_{s} \alpha(t) \perp \V_{s} = [\V_{s-1}, \basisvec_{s}]$.
This lets us express the error
\begin{align*}
    \|\V_{s-1} \widehat{\x}_{s-1}(t) - \x(t)\|_2^2 = \|\widehat{\x}_{s-1}(t) - \mathbf{p}_{s-1}(t)\|_2^2 + \alpha(t)^2 + \|\mathbf{q}(t)\|_2^2.
\end{align*}
For the \ac{rom} solution $\widehat{\x}_{s}^{(0)}(t) \in \R^{s}$, the structure of the expanded operators $\widehat{\mathbf{c}}_{s}^{(0)}, \widehat{\mathbf{A}}_{s}^{(0)}, \widehat{\mathbf{H}}_{s}^{(0)}$ implies
\begin{align*}
    \widehat{\x}_{s}(t) = [\widehat{\x}_{s-1}(t)\tr, \basisvec_{s}\tr \x_0]\tr = [\widehat{\x}_{s-1}(t)\tr, \alpha(0)]\tr.
\end{align*}
We can thus write its error to the \ac{fom} solution $\x(t)$ in the form
\begin{align*}
    \|\V_{s} \widehat{\x}_{s}^{(0)}(t) - \x(t)\|_2^2 
    &= \|\V_{s-1} \widehat{\x}_{s-1}(t) - \mathbf{p}_{s-1}(t)\|_2^2 + (\alpha(t)-\alpha(0))^2 + \|\mathbf{q}(t)\|_2^2 \\
    &= \|\V_{s-1} \widehat{\x}_{s-1}(t) - \x(t)\|_2^2 + (\alpha(t)-\alpha(0))^2 - \alpha(t)^2.
\end{align*}
The result follows.
\end{proof}

Proposition \ref{thm:expansionerror} shows that when the new initial guesses $\widehat{\mathbf{c}}_{s}^{(0)}, \widehat{\mathbf{A}}_{s}^{(0)}, \widehat{\mathbf{H}}_{s}^{(0)}$ are computed during iteration $s$ of Algorithm \ref{alg:flower}, the error $\|\V_s \widehat{\x}_s^{(0)} - \x(t)\|$ committed by the corresponding \ac{rom} at any time step $t$ is equal to the one committed by the optimal \ac{rom} chosen in the previous iteration plus the term $\zeta_s(t) = (\basisvec_{s}\tr \x_0)(\basisvec_{s}\tr \x_0 - 2 \basisvec_{s}\tr \x(t))$.
While $\zeta_s(t)$ is not guaranteed to be negative for all $t > 0$, at initial time $\zeta_s(0) = - (\basisvec_{s}\tr \x_0) \le 0$ and $\widehat{\x}_{s}^{(0)}(t)$ thus performs at least as well as $\widehat{\x}_{s-1}(t)$ close to $t=0$.
Moreover, because $\zeta_s(t)$ scales with $\basisvec_{s}\tr \x_0$, Proposition \ref{thm:expansionerror} implies in particular that if $\x_0 \perp \basisvec_{s}$ then $\|\V_{s} \widehat{\x}_{s+1}^{(0)}(t) - \x(t)\| = \|\V_{s-1} \widehat{\x}_{s-1}(t) - \x(t)\|$, i.e., no error is introduced by the expansion of the reduced operators $\widehat{\mathbf{c}}_{s-1}, \widehat{\mathbf{A}}_{s-1}, \widehat{\mathbf{H}}_{s-1}$ to $\widehat{\mathbf{c}}_s^{(0)}, \widehat{\mathbf{A}}_s^{(0)}, \widehat{\mathbf{H}}_s^{(0)}$.

Even with $\basisvec_{s}\tr \x_0 \neq 0$, because the snapshot reconstruction error $\sum_{k=1}^K \|\V_{s} \widehat{\x}_{s}^{(0)}(t_k) - \x(t_k)\|_2^2$ is used as part of the optimization over the regularization values, Proposition \ref{thm:expansionerror} gives a worst-case guarantee for how much the snapshot reconstruction error may change during iteration $s$:
Applying Proposition \ref{thm:expansionerror} to Remark \ref{rmk:minerrorreduction} yields
\begin{equation}\label{eq:doalmostnoharm}
    \begin{aligned}
    \delta_s^{*} &\le \delta_s^{(0)} = \delta_{s-1}^{*} + \sum_{k=1}^{K} \zeta_s(t_k)
    = K(\basisvec_s \tr \x_0)^2 - 2 (\basisvec_s\tr\x_0) \sum_{k=1}^{K} \basisvec_s\tr \x(t_k) \\
    &\le K(\basisvec_s \tr \x_0)^2 + 2 \sqrt{K} |\basisvec_s\tr\x_0| \sqrt{\sum_{k=1}^{K} (\basisvec_s\tr \x(t_k))^2} \\
    &= K(\basisvec_s \tr \x_0)^2 + 2 \sqrt{K} |\basisvec_s\tr\x_0| \sqrt{\varepsilon_s^2 - \varepsilon_{s-1}^2} \\
    &\le 3K (\basisvec_s \tr \x_0)^2 + \varepsilon_s^2 - \varepsilon_{s-1}^2,
\end{aligned}
\end{equation}
where we used Cauchy's inequality, the definition of the \ac{pod} modes in \eqref{eq:POD:minimization}, and Young's inequality.
While \eqref{eq:doalmostnoharm} is an over-approximation (if $\basisvec_s \tr \x_0 = 0$ then $\delta_s^{(0)} = \delta_{s-1}^{*}$), it gives an intuition for the worst-case error increase from one iteration of Algorithm~\ref{alg:flower} to the next.
Moreover, all terms within the upper bound in \eqref{eq:doalmostnoharm} are independent of the learned \ac{rom} and can be computed from the snapshots and basis vectors alone.
As such, \eqref{eq:doalmostnoharm} can be used to identify iterations of Algorithm \ref{alg:flower} where careful regularization is particularly important.

Applying Proposition \ref{thm:expansionerror} recursively shows that the snapshot reconstruction error with the \ac{rom} solution to \eqref{eq:rom:s} with the matrices $\opc_r^{(0)}$, $\opA_r^{(0)}$, $\opH_r^{(0)}$ is at least as good as with the choices $\opc_r = \mathbf{0}_{r \times 1}$, $\opc_r = \mathbf{0}_{r \times r}$, $\opc_r = \mathbf{0}_{r \times r^{(2)}}$:

\begin{corollary}\label{thm:comparisonzero}
    Suppose Algorithm \ref{alg:flower} was run with $r_0 = 1$, $\opc_1 = \opA_1 = \opH_1 = \mathbf{0}_{1 \times 1}$, and arbitrary $\overline{\delta} \ge 0$.
    Let $\x(t_1), \dots, \x(t_K)$ be the \ac{fom} solution snapshots, and $\widehat{\x}_r^{(0)} : [0, t_K] \rightarrow \R^r$ the reduced-order solution to \eqref{eq:rom:s} with the operators $\opc_r^{(0)}$, $\opA_r^{(0)}$, $\opH_r^{(0)}$ computed in iteration $s=r$ of Algorithm \ref{alg:flower}.
    Let $\overline{\x}_r : [0, t_K] \rightarrow \R^r$ be the solution to \eqref{eq:rom:s} with the operators $\opc_r = \mathbf{0}_{r \times 1}$, $\opA_r = \mathbf{0}_{r \times r}$, $\opH_r = \mathbf{0}_{r \times r^{(2)}}$.
    Then
    \begin{align}
        \sum_{k=1}^{K} \|\V_r \widehat{\x}_r^{(0)}(t_k) - \x(t_k) \|_2^2 \le \sum_{k=1}^{K} \|\V_r \overline{\x}_r(t_k) - \x(t_k) \|_2^2.
    \end{align}
    The two sums are equal if and only if $\opc_r^{(0)} = \mathbf{0}_{r \times 1}$, $\opA_r^{(0)} = \mathbf{0}_{r \times r}$, $\opH_r^{(0)} = \mathbf{0}_{r \times r^{(2)}}$.
\end{corollary}

\begin{proof}
    The proof is provided in Appendix \ref{sec:proofs}.
\end{proof}

Corollary \ref{thm:comparisonzero} states that the snapshot reconstruction accuracy with the \ac{rom} solution $\widehat{\x}_r^{(0)}$ is at least as good as when the operators $\opc_r, \opA_r, \opH_r$ are chosen as zero-matrices. 
The corollary is pessimistic in that the proof bounds $\delta_s^{*} \le \delta_s^{(0)}$ for each iteration of the Algorithm, i.e., it considers the worst-case scenario where \ac{opinf} fails to reduce the error in \textit{all} iterations $s=1, \dots, r-1$, and Algorithm \ref{alg:flower} goes into iteration $s=r$ with the trivial $\opc_r^{(0)} = \mathbf{0}_{r \times 1}$, $\opA_r^{(0)} = \mathbf{0}_{r \times r}$, $\opH_r^{(0)} = \mathbf{0}_{r \times r^{(2)}}$.
If that is the case, we recover the standard \ac{opinf} minimization problem as a worst-case guarantee;
otherwise, the matrices $\opc_r^{(0)}$, $\opA_r^{(0)}$, $\opH_r^{(0)}$ used as regularizers in the last iteration are for a \ac{rom} that performs strictly better in terms of snapshot reconstruction accuracy than the one associated to the typical Tikhonov regularization towards zero.
While this does not guarantee that the learned nested \ac{rom} with optimal weights outperforms the one learned with standard \ac{opinf}, it is still a good indicator. 
Moreover, the modeler can always compute the standard regularized \ac{opinf} model for comparison as a ``do-no-harm" guarantee.

\subsection{Stability of matrix updates}\label{sec:analysis:stability}


In iteration $s$ of Algorithm \ref{alg:flower}, we compute the operator candidates $\opc_s^{(i)}, \opA_s^{(i)}, \opH_s^{(i)}$ for regularization weights $\omega_c^{(i)}, \omega_A^{(i)}, \omega_H^{(i)}$ by solving the regularized least squares minimization problem \eqref{eq:opinf:regularized:initialguess}.
This is equivalent to solving for the updates $\opc_s^{(i)} = \opc_s^{(0)} + \Delta^c_s$, $\opA_s^{(i)} = \opA_s^{(0)} + \Delta^A_s$, $\opH_s^{(i)} = \opH_s^{(0)} + \Delta^H_s$ with
\begin{equation}\label{eq:opinf:regularized:zero}
    \begin{aligned}
    \min_{\Delta^c_s, \Delta^A_s, \Delta^H_s} &\|\Datamatrix_s \left[\begin{array}{l}
        (\Delta^c_s)\tr\\ 
        (\Delta^A_s)\tr\\ 
        (\Delta^H_s)\tr
    \end{array}\right] - \Rhsmatrix_{\Delta} \|_F^2
    + \omega_c^2\| \Delta^c_s\|_F^2 + \omega_A^2\|\Delta^A_s\|_F^2 + \omega_H^2\|\Delta^H_s\|_F^2
\end{aligned}
\end{equation}
and residual matrix 
$\Rhsmatrix_{\Delta} := \Rhsmatrix_s - \Datamatrix_s [\opc_s^{(0)}, \opA_s^{(0)}, \opH_s^{(0)}]\tr \in \R^{K \times s}$.
The norm of the updates is bounded from above by
\begin{align}\label{eq:update:bound}
    \| \text{row}_j(\Delta^c_s)\|_2^2 + \|\text{row}_j(\Delta^A_s)\|_2^2 + \|\text{row}_j(\Delta^H_s)\|_2^2 &\le \frac{1}{\sigma_s^{(i)}} \|\text{col}_j(\Rhsmatrix_{\Delta})\|_2^2 & j = 1, \dots, s,
\end{align}
where $\sigma_s^{(i)}$ is the smallest singular value of the extended data matrix 
$$[\Datamatrix_s\tr, \text{diag}(\mathbf{w})]\tr \in \R^{(K+s_{\rm{tot}} \times s_{\rm{tot}})}$$ with weight vector $\mathbf{w}^{(i)}_s = [\omega_{c}^{(i)}, \omega_{A}^{(i)} \mathbf{1}_{1 \times s}, \omega_{H}^{(i)} \mathbf{1}_{1 \times s^{(2)}}] \in \R^{s \times s_{\rm{tot}}}$ computed in line \ref{alg:flower:def:w} of Algorithm \ref{alg:flower}.
The singular value $\sigma_s^{(i)}$ is bounded from below by $\sigma_s^{(i)} \ge \sigma_{\min > 0}(\Datamatrix_s) + \min\{\omega_c^{(i)}, \omega_A^{(i)}, \omega_H^{(i)}\} >0$, where $\sigma_{\min > 0}(\Datamatrix_s)$ is the smallest, positive singular value of $\Datamatrix_s$.
This bound suggests that once $s$ is large enough for $\Datamatrix_s$ to be rank-deficient with $\sigma_{\min > 0}(\Datamatrix_s) \approx 0$, the updates $\Delta^c_s, \Delta^A_s, \Delta^H_s$ may increasingly be affected by numerical noise.
For this reason, we maximize over $\sigma_s^{(i)}$ to choose the regularization weights.

The bound \eqref{eq:update:bound} can be constructed analogously for the solution to the standard regularized \ac{opinf} minimization problem (of dimension $s$) with the residual matrix $\Rhsmatrix_{\Delta}$ replaced by $\Rhsmatrix_s$ but the same scaling coefficient $1/\sigma_s^{(i)}$. 
However, by construction, the norm of the residual decreases in the course of Algorithm \ref{alg:flower}:
For $1 \le j \le s - 1$, 
\begin{align}\label{eq:residual:change}
    \|\text{col}_j(\Rhsmatrix_{\Delta})\|_2^2 = \|\text{col}_j(\Rhsmatrix_{r}) - \Datamatrix_{s-1} [\opc_{s-1}, \opA_{s-1}, \opH_{s-1}]\tr\|_2^2
\end{align}
i.e., the $j$-th column of the residual with the initial guess in iteration $s$ is equal to the $j$-th column of the optimized residual at the end of iteration $s-1$. 
In combination with \eqref{eq:update:bound}, this equality implies that if any iteration $s$ successfully reduced its residual $\|\text{col}_j(\Rhsmatrix_{\Delta})\|_2^2$ to zero, it remains zero in all following iterations and the rows of the learned operators will not be needlessly updated.
This result is particularly important when warm-starting Algorithm \ref{alg:flower} with varying model approximations as it guarantees that no overly complicated model is learned.

The last column of the residual matrix is invariant under any previous iterations of Algorithm \ref{alg:flower}, i.e., $\text{col}_s(\Rhsmatrix_{\Delta}) = \text{col}_s(\Rhsmatrix_r)$.
However, its norm is bounded
\begin{align*}
    \|\text{col}_s(\Rhsmatrix_{\Delta})\|^2 
    &\approx  \frac{1}{\Delta t} \|\basisvec_s\tr(\x(t_0)-\x(t_1))\|_2^2 + \frac{1}{\Delta t} \sum_{k=2}^K \|\basisvec_s\tr(\x(t_k)-\x(t_{k-1}))\|_2^2 + \mathcal{O}(\Delta t) \\
    &\le \frac{6}{\Delta t} \sum_{k=1}^K \|\basisvec_s\tr \x(t_k)\|_2^2 + \mathcal{O}(\Delta t) \\
    &= \frac{6}{\Delta t}(\varepsilon_s^2 - \varepsilon_{s-1}^2) + \mathcal{O}(\Delta t),
\end{align*}
where we have used forward and backward finite differences to approximate the time-derivative, Young's inequality, and the definition \eqref{eq:POD:minimization} of the \ac{pod} error. 
Because $\varepsilon_s^2 - \varepsilon_{s-1}^2 \rightarrow 0$ for $s \rightarrow \infty$, the norm $\|\text{col}_s(\Rhsmatrix_{\Delta})\|^2$ decays asymptotically. 

In combination with \eqref{eq:residual:change}, we conclude that, asymptotically, the absolute change in residual caused by the learned operator updates decreases for all candidate regularization parameters.
This means that the actions of the learned operators on the projected snapshots become similar for the different regularization parameters, and vary primarily in the orthogonal complement $\V_s^{\perp}$. 
Consequently, the snapshot reconstruction error becomes similar, and, as $s$ increases, it becomes increasingly likely that Algorithm \ref{alg:flower} chooses a stronger regularization without sacrificing reconstruction accuracy.


\section{Extensions}\label{sec:extensions}

In this section, we briefly discuss extensions and use cases of Algorithm \ref{alg:flower} that are helpful in practice.

\subsection{Individualized regularization}\label{sec:block-regularization}

For simplicity, in Algorithm \ref{alg:flower}, we apply the same regularization to each entry of any one operator update $\Delta_s^{c}, \Delta_s^{A}, \Delta_s^{H}$.
However, this choice does not reflect how much trust we have in the entries of $\opc_{s-1}, \opA_{s-1}, \opH_{s-1}$ from the previous iteration. 
As updates become smaller with increasing $s$, the interactions between the most important modes can increasingly be trusted and thus be regularized more strongly.
In fact, too little regularization can cause the quality for the \ac{rom} defined on the subspaces $\V_{s-1}, \V_{s-2}, ...$ to deteriorate.
In contrast, updates to the entries for the higher-order modes tend to be larger, partially due to the encoded higher-frequency components, and they take more iterations to converge.
To reflect this trust, one can introduce individual weights for each update entry that is based on, for example, how many iterations ago its first approximation was learned. 
This way, stronger regularization can be applied to the earlier entries.
This comes, however, at the cost of identifying a suitable weight update strategy.

A compromise is to specify --- for each operator --- one weight for all entries that were learned in previous iterations and a smaller weight for all new operator entries.
This strategy tends to increase the smallest singular value of the regularized data matrix more strongly than if regularization is applied evenly to all operator updates; the downside is the more expensive optimization over six instead of three regularization parameters.

The extreme case of this strategy is to enforce the entries from previous iterations. 
This reduces the amount of unknowns learned in each iteration from $s s_{\rm{tot}} = s + \frac32 s^2 + \frac12 s^3$ to $s_{\rm{tot}} + (s-1)(1+s) =  \frac32 s + \frac32 s^2$.
Because the corresponding data matrix is skinnier for the first $s-1$ test functions, its smallest singular value is larger than when all entries are updated together.
Moreover, by enforcing the entries learned in previous iterations, when restricting the learned \ac{rom} for $\V_r$ back onto the subspaces $\V_s$ with $s < r$, one recovers the same model that was chosen during iteration $s$ of Algorithm \ref{alg:flower} and its accuracy.
The disadvantage of this approach is that it can easily lead to overfitting, with the size of new entries growing exponentially.
It thus needs to be used carefully, e.g., with a fall-back option if the error does not behave as expected.

\subsection{Iterative updates}\label{sec:iterative}

While we previously primarily analyzed the later iterations $1 \ll s \le r$ of Algorithm \ref{alg:flower} (Section \ref{sec:analysis:stability}), we now focus on the earlier iterations $1 \le s \ll r$.
The entries learned early are associated to the most dominant \ac{pod} modes and thus determine the majority of the snapshot reconstruction error of the learned \ac{rom}, regardless of dimension.
Moreover --- because Algorithm \ref{alg:flower} is iterative --- they influence the operators learned in all subsequent iterations.
However, in those early iterations, the residual $\Rhsmatrix_s - \Datamatrix_{s-1} [\opc_{s-1}, \opA_{s-1}, \opH_{s-1}]\tr$ is typically still large because the operators are not yet expressive enough to capture all details in $\Rhsmatrix_s$. 
As a consequence, the \ac{rom} \eqref{eq:rom:s} associated to learned operators $\opc_s, \opA_s, \opH_s$ might perform poorly.
Because the derivatives $\dot{\mathbf{p}}_s(t_k) = \V_s\tr \dot{\x}(t_k)$ are matched only in a least squares sense with a large overall residual, the \ac{rom} solution $\widehat{\x}_s(t)$ may diverge early from the intended trajectory $\mathbf{p}(t)$.
This behavior is correctable by including the \ac{rom} solution into the \ac{opinf} learning problem via exchanging the call to \texttt{OpInf} (Algorithm \ref{alg:opinf}) in line \ref{alg:flower:def:learn} of Algorithm \ref{alg:flower} to a call to \texttt{IterativeUpdates}, Algorithm \ref{alg:iterative}.

\begin{algorithm}
\caption{IterativeUpdates}\label{alg:iterative}
\begin{algorithmic}[1]
\smallskip
\State \textbf{Input:}
reduced dimension $s$,
data matrix $\Datamatrix_s \in \R^{K \times s_{\rm{tot}}}$, 
time derivative matrix $\Rhsmatrix_s\in \R^{K \times s}$, 
weight vector $\mathbf{w} \in \R^{s_{\rm{tot}}}_{>0}$,
initial guesses $\opc_{\rm{itr}} \in \R^s, \opA_{\rm{itr}} \in \R^{s \times s}, \opH_{\rm{itr}} \in \R^{s \times s^{(2)}}$, projected initial condition $\mathbf{p}_s(t_1) \in \R^s$
\State \textbf{Hyper-parameters:} repeat threshold $i_{\rm{max}} \in \mathbb{N}$
\State \textbf{Output:} Reduced operators $\opc_{s} \in \R^{s}$,
$\opA_{s} \in \R^{s \times s}$,
$\opH_{s} \in \R^{s \times s^{(2)}}$, minimal encountered singular value $\sigma_{\min}$
\Statex \hrulefill
\State Initialize $\sigma_{\min} \gets \infty$
\smallskip
\Statex \texttt{\# Iterate over the learned operators}
\For{$i = 1, \dots, i_{\max}$}
\smallskip
\Statex \quad \texttt{\# Solve current ROM}
\State Solve $\dot{\widehat{\x}} = \opH_{\rm{itr}} [\widehat{\x} \widehat{\otimes} \widehat{\x}] + \opA_{\rm{itr}}\widehat{\x} + \opc_{\rm{itr}}$ ($0 < t < t_K$), $\widehat{\x}(0) = \mathbf{p}_s(t_1)$
\State $\widehat{\mathbf{X}}_{\rm{itr}} \gets [\widehat{\x}(t_1), \dots, \widehat{\x}(t_K)] \in \R^{s \times K}$ 
\smallskip
\Statex \quad \texttt{\# Build matrices}
\State $\Datamatrix_{\rm{itr}} \gets [\mathbf{1}_{K \times 1}, \widehat{\mathbf{X}}_{\rm{itr}}\tr, [\widehat{\mathbf{X}}_{\rm{itr}} \odot \widehat{\mathbf{X}}_{\rm{itr}}]\tr \mathbf{C}_s]$
\State Expand data matrix $\Datamatrix \gets [\Datamatrix_s\tr, \Datamatrix_{\rm{itr}}\tr]\tr \in \R^{2K \times s_{\rm{tot}}}$
\State Expand time-derivative matrix $\Rhsmatrix \gets [\Rhsmatrix_s\tr, \Rhsmatrix_s\tr]\tr \in \R^{2K \times s}$
\smallskip
\Statex \quad \texttt{\# Solve least squares problem}
\State $\opc_{\rm{itr}}, \opA_{\rm{itr}}, \opH_{\rm{itr}}, \sigma \gets \texttt{OpInf}(\Datamatrix, \Rhsmatrix, \mathbf{w}, \opc_{\rm{itr}}, \opA_{\rm{itr}}, \opH_{\rm{itr}})$
\State $\sigma_{\min} \gets \min\{\sigma_{\min}, \sigma\}$
\smallskip
\EndFor \\
\smallskip
\Return $\opc_{\rm{itr}}, \opA_{\rm{itr}}, \opH_{\rm{itr}}, \sigma_{\min}$
\smallskip
\end{algorithmic}
\end{algorithm}

Algorithm \ref{alg:iterative} updates, in each iteration $i$, the operators $\opc_{\rm{itr}}^{(i-1)}, \opA_{\rm{itr}}^{(i-1)}, \opH_{\rm{itr}}^{(i-1)}$ by solving the regularized least squares problem
\begin{equation}\label{eq:opinf:iterative}
    \begin{aligned}
    \min_{\opc_s, \opA_s, \opH_s} &\|\Datamatrix_s \left[\begin{array}{c}
        \opc_s\tr\\ 
        \opA_s\tr\\ 
        \opH_s\tr
    \end{array}\right] - \Rhsmatrix_s \|_F^2 
    + \|\Datamatrix_{\rm{itr}} \left[\begin{array}{c}
        \opc_s\tr\\ 
        \opA_s\tr\\ 
        \opH_s\tr
    \end{array}\right] - \Rhsmatrix_s \|_F^2 
    \\
    &\qquad \qquad \qquad \qquad + \| \text{diag}(\mathbf{w}) \left[
    \begin{array}{l}
    \opc_s\tr-[\opc_{\rm{itr}}^{(i-1)}]\tr \\
    \opA_s\tr-[\opA_{\rm{itr}}^{(i-1)}]\tr \\
    \opH_s\tr-[\opH_{\rm{itr}}^{(i-1)}]\tr
    \end{array}
    \right]\|_F^2.
\end{aligned}
\end{equation}
The first and the third term, together, comprise the standard \ac{opinf} cost function \eqref{eq:opinf:regularized:initialguess} with regularization towards the iterates $\opc_{\rm{itr}}^{(i-1)}, \opA_{\rm{itr}}^{(i-1)}, \opH_{\rm{itr}}^{(i-1)}$ (instead of $\opc_s^{(0)}, \opA_s^{(0)}, \opH_s^{(0)}$).
In the second term, instead of using the \ac{fom} snapshots, the data matrix $\Datamatrix_{\rm{itr}}$ is constructed from the \ac{rom} solution $\widehat{\x}_{\rm{itr}}(t)$ to 
\begin{equation}\label{eq:rom:itr}
    \begin{aligned}
    \dot{\widehat{\x}}_{\rm{itr}}(t) 
    &= \opc_{\rm{itr}}^{(i-1)} + \opA_{\rm{itr}}^{(i-1)}\widehat{\x}_{\rm{itr}}(t) + \opH_{\rm{itr}}^{(i-1)}[\widehat{\x}_{\rm{itr}}(t) \widehat{\otimes} \widehat{\x}_{\rm{itr}}(t)], \quad t > 0
\end{aligned}
\end{equation}
with initial condition $\widehat{\x}_{\rm{itr}}(0) := \V_s\tr \widehat{\x}_0$.
The second term in \eqref{eq:opinf:iterative} hence penalizes updates to $\opc_{\rm{itr}}^{(i-1)}, \opA_{\rm{itr}}^{(i-1)}, \opH_{\rm{itr}}^{(i-1)}$ that would move the current \ac{rom} trajectory $\widehat{\x}_{\rm{itr}}$ further away from the ideal trajectory $\mathbf{p}_s(t)$, and encourages operators that brings it closer.
The distance between the trajectories is measured through the difference $\dot{\widehat{\x}}_{\rm{itr}} - \dot{\mathbf{p}}_s$ in time derivatives.
With this choice of measure\footnote{
Note that $\|\x\|^2 := \| \frac{d}{dt} \x\|^2_{L^2([0, T], \R^s)} + \| \x(0) \|_2^2$ indeed defines a norm on $C^1([0, T], \R^s)$, and that $\widehat{\x}_{\rm{itr}}(0) = \widehat{\mathbf{p}}_s(0) = \V_s\tr \x_0$.
}
we achieve that \eqref{eq:opinf:iterative} remains a linear least squares problem that can be solved with standard techniques.

\begin{remark}
    The idea presented here is similar to the roll-out \ac{opinf} method \cite{uy2023rollouts}, with the difference that \cite{uy2023rollouts} optimizes over the distance $\sum_{k=1}^K \| \widehat{\x}_{\rm{itr}}(t_k) - \mathbf{p}_s(t_k) \|_2^2$ in trajectories directly, resulting in a constrained optimization problem.
\end{remark}

The cost of running \texttt{IterativeLearning} (Algorithm \ref{alg:iterative}) instead of \texttt{OpInf} (Algorithm \ref{alg:opinf}) within \texttt{NestedOpInf} (Algorithm \ref{alg:flower}) increases its runtime by factor $i_{\max}$. 
This cost can be reduced by including a convergence criterion.
In our experience, a small $i_{\max}$ can already significantly improve the operators learned in the Algorithm \ref{alg:flower}. 
In our numerical experiments we use $i_{\max} = 5$ in Section \ref{sec:cubicheat} and $i_{\max} = 1$ in Section \ref{sec:greenland}.


\section{Results}\label{sec:results}


We provide two numerical examples.
In the first example, we contrast nested and standard \ac{opinf} using a 1-dimensional, cubic \ac{pde}.
The second example involves a large-scale model of the Greenland ice sheet.
The numerical experiment is designed to show how nested \ac{opinf} works in face of large model approximation errors.

\subsection{Cubic heat equation}\label{sec:cubicheat}

We consider the parameterized \ac{pde}
\begin{align}\label{eq:fom:cubicheat}
    \dot{x}(t, z) &= \kappa \Delta x(t, z) - x(t, z)^3, &t>0, z \in (0,1)
\end{align}
with initial condition $x(0, z) = 10z(1-z)$ and zero-Dirichlet boundary conditions $x(t, 0) = x(t, 1) = 0$ for all $t \ge 0$.
We consider parameters $\kappa \in [0.001, 0.1]$.
We discretize \eqref{eq:fom:cubicheat} in space with $n=1,001$ piecewise linear finite element basis functions, and compute the solution until $t=1$ using Crank-Nicolson time-stepping with a time step size $\Delta t = 0.001$.
On average, a solve takes \SI{8}{s}.
\footnote{We use the finite element library \texttt{FEniCSx} on a machine with 2.60GHz base speed and 2 TiB main memory across 256 cores. Computations are performed in sequence.} 

\begin{figure}
    \centering
    \includegraphics[width=\linewidth]{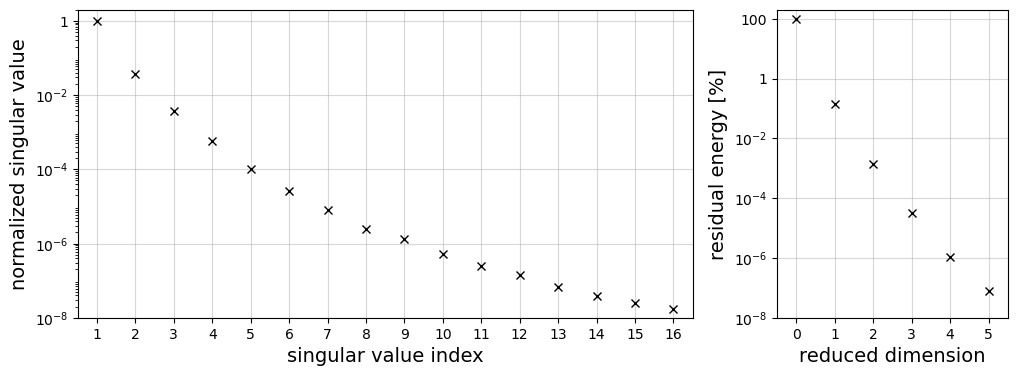}
    \caption{Singular value decay (left) and residual energy (right) for training snapshots}
    \label{fig:cubicheat:singularvalues}
\end{figure}

For our training data, we solve \eqref{eq:fom:cubicheat} for three training parameters $\kappa_1 = 0.1, \kappa_2 =0.01, \kappa_3 =0.001$ until $t=0.2$, for a total of $K = 603$ training snapshots.
Weighing the snapshots in the $L^2((0,1))$ inner product, we compute $r=5$ orthonormal basis functions, capturing 99.99999992\% of the snapshot energy.
The singular value decay is plotted in Figure \ref{fig:cubicheat:singularvalues}.

To train our nested \ac{opinf} \ac{rom}, we run Algorithm \ref{alg:flower} with $n_{\omega} = 24$ combinations of regularization parameters.
With $i_{\max} = 5$ iterations, its runtime is \SI{40}{s}.
To keep offline costs comparable, we train the standard \ac{opinf} \ac{rom} in a grid search with $n_{\omega} = 1,581$ regularization parameter combinations; this takes \SI{53}{s}.
On average, both \acp{rom} take below \SI{16}{ms} to run up to $t=1$, for a speed-up above $500\times$ compared to the \ac{fom}.

\begin{figure}
    \centering
    \includegraphics[width=\linewidth]{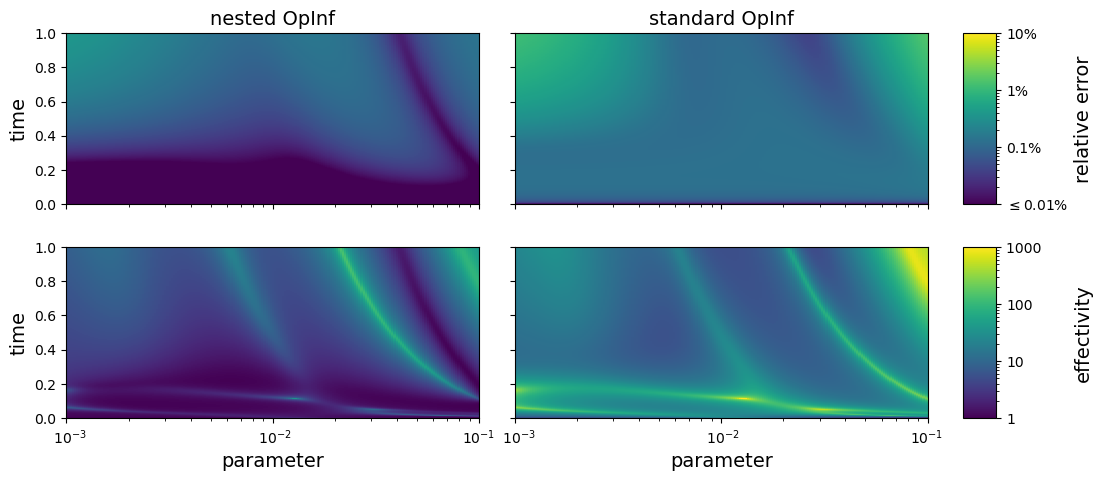}
    \caption{Relative error [\%] (top) and effectivity (bottom) of the standard and nested \ac{opinf} \acp{rom} over the parameter domain (x-axis) and time (y-axis).}
    \label{fig:cubicheat:error2D}
\end{figure}

\begin{table}[]
    \centering
    \begin{tabular}{ll|rr|rr|rr|rr}
    && \multicolumn{4}{c|}{training parameters} & \multicolumn{4}{c}{all parameters} \\
    && \multicolumn{4}{c|}{$\kappa \in \{0.001, 0.01, 0.1\}$} & \multicolumn{4}{c}{$\kappa \in [0.001, 0.1]$} \\\cmidrule{3-10}
    && \multicolumn{2}{c|}{training time} & \multicolumn{2}{c|}{full time} & \multicolumn{2}{c|}{training time} & \multicolumn{2}{c}{full time}\\
    && \multicolumn{2}{c|}{$t \in [0, 0.2]$} & \multicolumn{2}{c|}{$t \in [0, 1]$} & \multicolumn{2}{c|}{$t \in [0, 0.2]$} & \multicolumn{2}{c}{$t \in [0, 1]$}\\\cmidrule{3-10}
    & & mean & max & mean & max & mean & max & mean & max \\
    \hline
    error $[\%]$ &std. \ac{opinf} & 0.119 & 0.178 & 0.395 & 1.616 & 0.112 & 0.178 & 0.209 & 1.616 \\
     &nested \ac{opinf} & 0.004 & 0.008 & 0.088 & 0.397 & 0.005 & 0.030 & 0.070 & 0.397 \\
    &projection & 0.002 & 0.008 & 0.013 & 0.050 & 0.004 & 0.014 & 0.012 & 0.050 \\
    \hline
    effectivity & std. \ac{opinf} & 83 & 380 & 99 & 788 & 51 & 1825 & 33 & 1825 \\
     & nested \ac{opinf} & 2 & 7 & 12 & 91 & 2 & 35 & 8 & 136 \\
    \end{tabular}
    \caption{Mean and maximum relative error [\%] and effectivity of the standard and nested \ac{opinf} \acp{rom} for different parameter sets and time intervals. The values for $\kappa \in [0.001, 0.1]$ are computed over 201 evenly spaced parameters.}
    \label{tab:cubicheat:error}
\end{table}

The top row of Figure \ref{fig:cubicheat:error2D} shows the relative error of both \acp{rom} over time and parameter domain.
Most notably, the error incurred by the nested \ac{rom} remains below \SI{0.03}{\%} for the whole training time interval $[0, 0.2]$ for all parameters, with an average error of \SI{0.005}{\%}.
In comparison, the standard \ac{opinf} model incurs an average error of \SI{0.112}{\%}, with a maximum of $\SI{0.178}{\%}$.
Beyond the training time interval the error increases for both models, though the average error with the nested \ac{rom} remains smaller (\SI{0.07}{\%} compared to \SI{0.2}{\%}).
As the increase in error may be partially unavoidable due to the fixed reduced space $\V_r$, in the bottom row of Figure \ref{fig:cubicheat:error2D} we compare the effectivities of the \acp{rom}, which we define as the ratio 
\begin{align}
    \frac{\sqrt{\Delta t\sum_{k=1}^{k_{\max}} \|\V_r \widehat{\x}_r(t_k) - \x(t_k)\|^2}}{\sqrt{\Delta t\sum_{k=1}^{k_{\max}} \|\V_r \V_r\tr \x(t_k) - \x(t_k)\|^2}} \ge 1
\end{align}
of the \ac{rom} error to the projection error, with $k_{\max} = 201$ when on the training time interval $[0, 0.2]$ and $k_{\max} = 1001$ when comparing on the entire time interval $[0, 1]$.
Overall, we see that the effectivities of the nested \ac{opinf} \ac{rom} are smaller than for the standard \ac{rom}, with a mean of 8 and a maximum of 136 (compared to 33 and 1,825 for standard \ac{opinf}).
The mean and maximum values for the different training and testing time and parameter combinations are reported in Table \ref{tab:cubicheat:error} for reference.

\begin{figure}
    \centering
    \includegraphics[width=\linewidth]{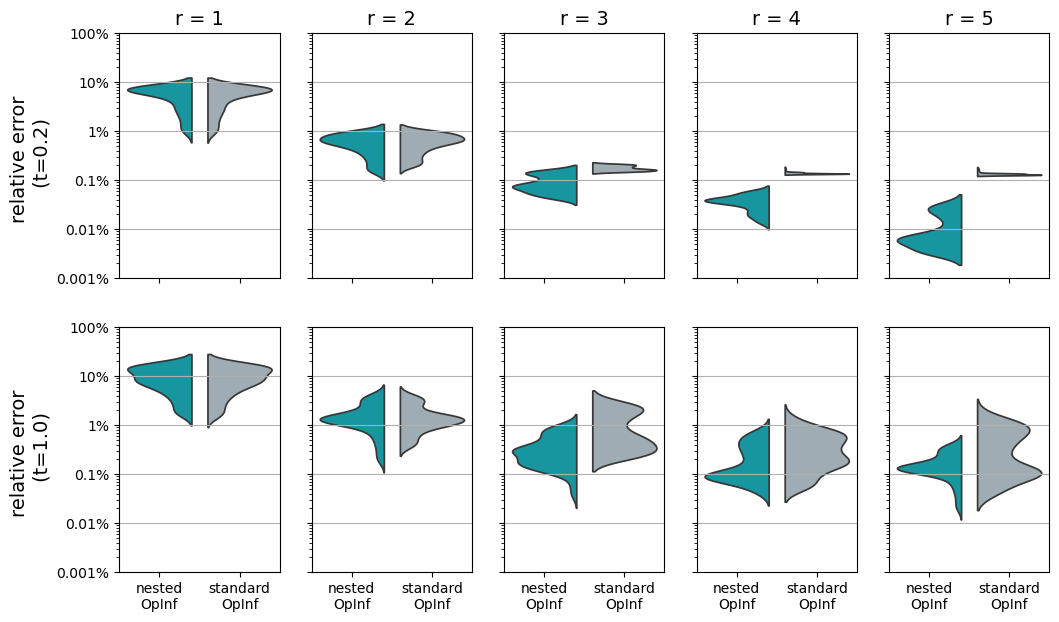}
    \caption{Distributions of the relative error [\%] over the parameter domain at the final training time $t=0.2$ (top row) and the final time $t=1.0$ for nested and standard \ac{opinf} by reduced dimension $r$.}
    \label{fig:cubicheat:error-by-r}
\end{figure}

We next compare the performance of nested and standard \ac{opinf} for different reduced dimensions.
To this end, Figure \ref{fig:cubicheat:error-by-r} compares the distribution of the relative error at times $t=0.2$ and $t=1.0$ over the parameter domain.
Here, the nested \ac{opinf} \acp{rom} use the operators $\opc_s$, $\opA_s$, $\opH_s$ ($s=1, \dots, 5$) obtained in the course of running Algorithm~\ref{alg:flower} for the previous results ($n_{\omega} = 24$).
The standard \ac{opinf} \acp{rom} for $r=1, \dots, 4$ were trained anew with regularization weights optimized individually for each \ac{rom} ($n_{\omega} = 1,581$);
the \ac{rom} for $r=5$ is the same as in the previous results.
As expected for $r=1$, the relative errors for both training methods are comparable, differing only marginally as a consequence to the different candidate regularization weights.
As $r$ increases, the informed initial guesses with nested \ac{opinf} increasingly yield benefits:
For $t=0.2$, the error with standard \ac{opinf} stagnates around \SI{0.2}{\%} when the data matrix $\Datamatrix_r$ becomes rank-deficient ($r \ge 3$), while the error with nested \ac{opinf} continues to decrease.
In the time-generalization to $t=1.0$, the difference between the two methods is less dominant, with nested \ac{opinf} still showing the smaller minimum, mean, and maximum errors for $r \ge 3$.

\subsection{Greenland ice sheet}
\label{sec:greenland}

\begin{figure}
    \centering
    \includegraphics[width=0.8\linewidth]{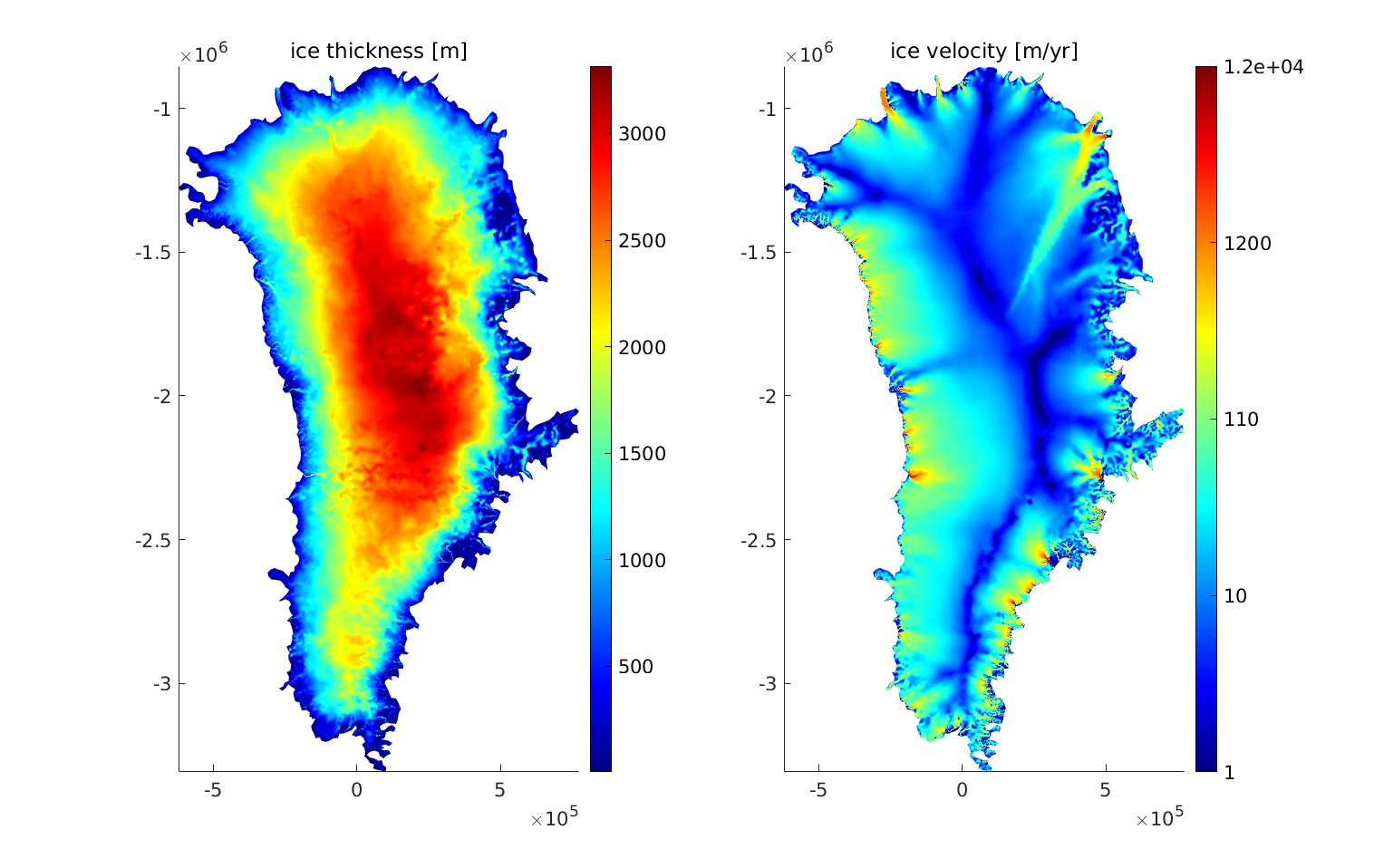}
    \caption{2015 ice thickness (left, in meters) and ice velocity (right, in meters per year) initial conditions at parameter $z = 0.5$.}
    \label{fig:Greenland:fields}
\end{figure}

In the following we build a \ac{rom} for the 2020-2050 change in ice thickness of the Greenland ice sheet under a varying parameter $z \in [0.2625, 0.7375]$ characterizing basal friction.
We model the ice thickness field $h$ and depth-averaged velocity field $\overline{\mathbf{v}}$ of the Greenland ice sheet under the (low-emission) shared economic pathway scenario SSP1-2.6 (\cite{RN1,RN3}, chapter 1.6).
The governing equation for the ice thickness is
\begin{align}\label{eq:Greenland:thickness}
    \dot{h}(t, \mathbf{y}; z) &= - \nabla \cdot (\overline{\mathbf{v}}(t, \mathbf{y}; z, h)h(t, \mathbf{y}; z)) + m_s(t, \mathbf{y}; h) &\forall ~ t > 2015, ~\mathbf{y} \in \Omega
\end{align}
with initial condition $h(2015, \,\cdot\,; z ) = h_0$ (shown in Figure \ref{fig:Greenland:fields}, left).
The field $m_s$ is the surface mass balance following the atmospheric protocol for Greenland in \cite{nowicki2020experimental}, and $\Omega \in \R^2$ is the domain of the Greenland ice sheet extracted from \cite{BedMachineV5, morlighem2017bedmachine} as described in \cite{aretz2025multifidelity}.
At any time $t$, the velocity is modeled through the Shallow-Shelf Approximation (\cite{MacAyeal1989}).
For conciseness, we refer to \cite{aretz2025multifidelity}, Sections 1 and 4.1 (model number 7), for a detailed description of our experimental protocol and the governing equations including the coupling to ice temperature.
Our model is discretized and run in the Ice-Sheet and Sea-Level System Model (\cite{issm}), a glaciology community code.
Using linear finite elements, the dimension for the discretized ice thickness $\mathbf{h}(t; z)$ and the two horizontal components of the depth-averaged velocity field is $n=20,455$ each.
Time step sizes are chosen adaptively to meet a CFL condition.

For any simulation year, the surface mass balance field $m_s$ provides the local rates of ice accumulation on the ice sheet's surface, as computed by CNRM-CM6-1 climate model \cite{voldoire2018cnrm}, with an altitude-dependent adjustment computed from the simulated ice thickness $\mathbf{h}$.
The velocity field $\overline{\mathbf{v}}$ depends on $\mathbf{h}$ through the slope of the ice surface and the effective pressure at the ice base.
Additionally, for our experimental setup here, $\overline{\mathbf{v}}$ depends on a deterministic parameter $z \in [0.2625, 0.7375]$, which maps onto the basal friction field $\alpha(z)$ in the sliding basal boundary condition for $\overline{\mathbf{v}}$.
Specifically, we choose $z$ to be the z-value of the dominant mode in the normal distribution of $\log(\alpha)$ in \cite{aretz2025multifidelity}, section 2.2.

We generate training trajectories of the ice thickness $\{\mathbf{h}(t_k^{(i)}; z_i)\}_{k=1}^{K_i}$ and depth-averaged ice velocity $\{\overline{\mathbf{v}}(t_k^{(i)}; z_i)\}_{k=1}^{K_i}$ for 17 equidistantly spaced training parameters $\{z_i\}_{i=1}^{17}$.
Note that the temporal discretizations $2015 = t_1^{(i)} < t_2^{(i)} < \dots < t_{K_i}^{(i)} = 2050$ differ between training parameters $z_i$ because of the adaptive time-stepping; however, each temporal discretization includes the integers $2015, 2017, \dots, 2050$ where the surface mass balance changes.
After dropping the 5-year burn-in time\footnote{The burn-in is an initial time period in which the model adjusts to changed parameterizations.}, we compute the training snapshot $\x(t_k^{(i)}; z_i) := \mathbf{h}(t_k^{(i)}; z_i) - \mathbf{h}(2020; z_i)$, $1 \le k \le K_1$ and $1\le i \le 17$, as the ice thickness change compared to the year 2020.
In total, we obtain $K = 6,976$ training snapshots.
For testing, we generate trajectories of the ice thickness for 240 additional parameters, spaced equidistantly between the original training parameters such that there is no overlap between the training and the testing parameter sets.
On average, the computation of a training or testing snapshot takes \SI{21}{min} of CPU time (run in parallel with 2 cores per snapshot on a machine with 2.60GHz base speed and 2 TiB main memory across 256 cores).

\begin{figure}
    \centering
    \includegraphics[width=\linewidth]{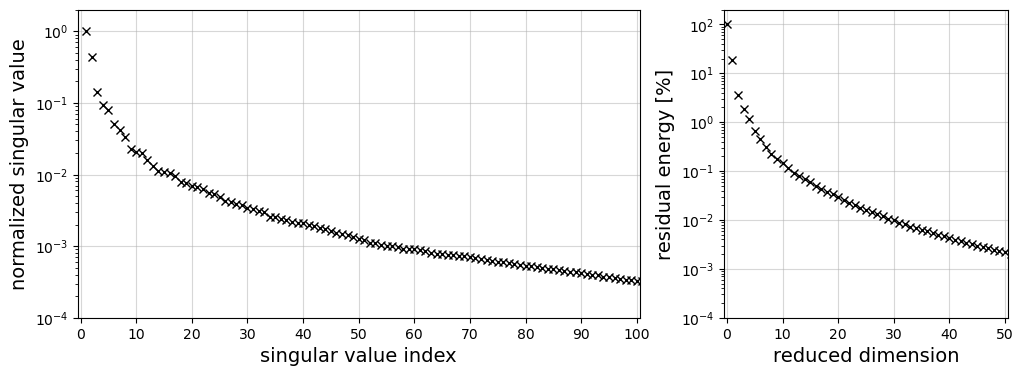}
    \caption{Singular value decay (left) and residual energy (right) for training snapshots for the Greenland ice sheet data.}
    \label{fig:Greenland:singularvalues}
\end{figure}

Figure \ref{fig:Greenland:singularvalues} shows the singular value decay and the residual energy for our training data.
We  build a \ac{pod} space $\V_r$ of reduced dimension $r = 30$, capturing \SI{99.990}{\%} of the residual energy.
Because the ice thickness equation \eqref{eq:Greenland:thickness} is strongly convection-dominated, the singular value decay is slow, and the reduced space $\V_r$ does not generalize well beyond the final training time $T=2050$.
To capture the influence of the parameter $z$ through the initial velocity at $t=2020$, we build a 4-dimensional reduced space $\mathbf{W} = [\mathbf{w}_1, \dots, \mathbf{w}_4] \in \R^{2n \times 4}$ for $\{\overline{\mathbf{v}}(2020; z_i)\}_{i=1}^{17}$ with \ac{pod}.
The space $\mathbf{W}$ captures \SI{99.843}{\%} of the velocity snapshot energy.
With these spaces, 
we build a parameterized \ac{rom} of the form
\begin{align}\label{eq:Greenland:rom}
    \dot{\widehat{\x}}_r (t; z) &= \widehat{\mathbf{m}}_r(t) + \opc_r(z) + \widehat{\mathbf{A}}_r(z)\widehat{\x}(t; z), &t > 2050
\end{align}
with initial condition $\widehat{\x}(2020; z) = \mathbf{0} \in \R^r$.
Here, we model the annual changes in the surface mass balance through
\begin{align}\label{eq:Greenland:smb}
    \widehat{\mathbf{m}}_r(t) = \sum_{j=2020}^{2049} \mathbbm{1}_{(j, j+1]}(t) \widehat{\mathbf{m}}_{r,j}
\end{align}
where, for $2015 \le j \le 2049$, $\mathbbm{1}_{(j, j+1]}(t)$ is the indicator function for the time interval $(j, j+1]$, and $\widehat{\mathbf{m}}_{r,j} \in \R^r$ is to be learned with \ac{opinf}.
We note that the characterization \eqref{eq:Greenland:smb} captures the annual nature of the atmospheric protocol \cite{nowicki2020experimental}, but does not specifically address altitude-dependent adjustments.
For the constant and linear operators $\opc_r(z)$ and $\opA_r(z)$, we impose affine decompositions of the form
\begin{align}\label{eq:Greenland:affine}
    \opc_r(z) &= \sum_{j=0}^{4} \theta_j(z) \opc_{r,j}, &
    \opA_r(z) &= \sum_{j=0}^{4} \theta_j(z) \opA_{r,j},
\end{align}
where $\{\opc_{r,j}\}_{j=0}^5 \subset \R^{r}$ and $\{\opA_{r,j}\}_{j=0}^5 \subset \R^{r \times r}$ are to be learned with \ac{opinf}.
We fix the first scaling coefficient $\theta_0(z) := 1$ independently of $z$, and choose, for $1 \le j \le 4$ and each training parameter $z_i$, $\theta_j(z_i) := \mathbf{w}_j\tr \overline{\mathbf{v}}(2020; z_i)$ as the projection of the 2020 training velocities $\overline{v}(2020; z_i)$ onto the $j$-th mode $\mathbf{w}_j$ of the reduced velocity space.
The representations \eqref{eq:Greenland:affine} are derived from \eqref{eq:Greenland:thickness} by (1) approximating the velocity first as a constant in time that only depends on the parameter $z$ and then further in the reduced space $\mathbf{W}$, (2) accounting for the centering of the ice thickness $\mathbf{h}(t; z)$ by $\mathbf{h}(2020; z)$; and (3) accounting for a small diffusion term that the \ac{fom} code introduces for numerical stabilization.
In total, we learn 35 vectors of dimension $r$ and five matrices of dimension $r \times r$ for our \ac{rom} \eqref{eq:Greenland:rom}, for a total of 550 degrees of freedom.
The corresponding data matrix $\Datamatrix_r \in \R^{6976 \times 185}$ is rank deficient with smallest positive singular value $\sigma_{\min>0}(\Datamatrix_r) = 1.19 \cdot 10^{-5}$ and largest singular value $\sigma_{\max}(\Datamatrix_r) = 124$.

We train our \ac{rom} using Algorithm \ref{alg:flower} with $n_{\omega} = 7$ candidate regularization weights, which --- following Section \ref{sec:block-regularization} --- we increase by factor 10 for entries learned in previous iterations.
Moreover, we call Algorithm \ref{alg:iterative} (with $i_{\max} = 1$) in line \ref{alg:flower:def:learn} as described in Section \ref{sec:iterative}.
The training time is \SI{18.4}{min} (not including the snapshot generation and the construction of the reduced space $\V_r$).
We run the \ac{rom} \eqref{eq:Greenland:rom} using implicit Euler time stepping with fixed $\Delta t = 0.01$, and, to evaluate errors to the full-order trajectory, interpolate the \ac{rom} solution onto the time discritization chosen by the \ac{fom}.
At any unseen parameter $z \in (z_i, z_{i+1})$, we interpolate linearly between the scaling coefficients $\theta_j(z_i)$ and $\theta_j(z_{i+1})$ of the neighboring two training parameters. 
The runtime of the \ac{rom} is \SI{0.06}{s} on average, for a minimum computational speed-up of $12,065\times$, and an average speed-up of $19,717\times$.

\begin{figure}
    \centering
    \includegraphics[width=\linewidth]{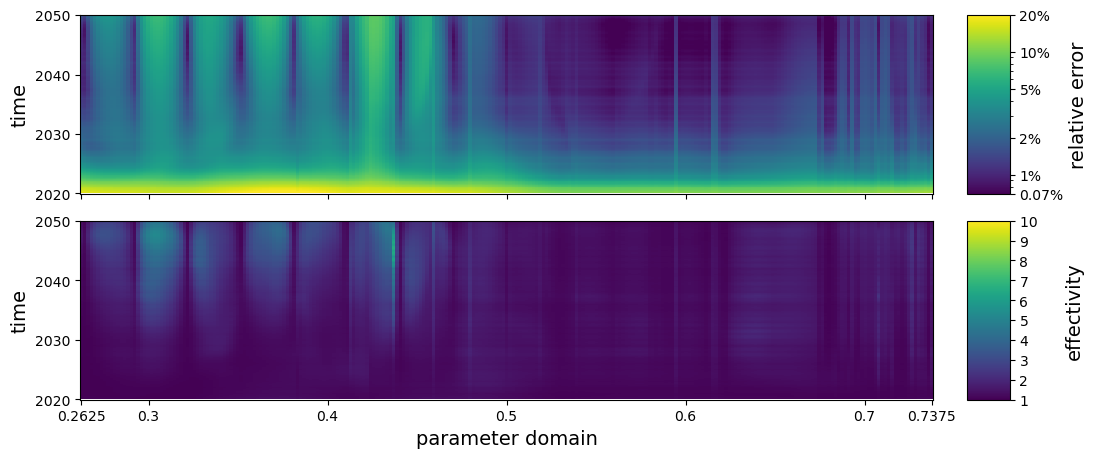}
    \caption{Relative error (top) and effectivity (bottom) in time (y-axis) over the parameter domain (x-axis).}
    \label{fig:greenland:2D}
\end{figure}

Figure \ref{fig:greenland:2D} (top) shows the distribution of the \ac{rom}'s relative error over the time and parameter domains.
Most notably, the error is worst around the initial time $t=2020$.
This behavior expected because we trained the reduced space with respect to the absolute error; close to $t=2020$, the snapshots' norms are smaller which causes them to be represented less in the reduced space.
This explanation is supported further by the near-optimal effectivity --- the ratio of the \ac{rom}'s error to the projection error --- between $t=2020$ and $t=2025$.
Beyond $t=2025$, the error stabilizes, increasing only slowly in time.
Despite the large error around $t=2020$, the time-averaged relative error lies between \SI{1.651}{\%} and \SI{3.288}{\%} for the training parameters, and between \SI{1.663}{\%} and \SI{7.542}{\%} for the test parameters, with \SI{2.365}{\%} and \SI{3.136}{\%} on average.
This generalization error from training to test parameters is reasonable because, with the effectivity below 10 for all parameters and all time steps, the \ac{rom} performs almost as good as the reduced space allows, despite the introduced model approximation error in the ice velocity and surface mass balance.


\section{Conclusion}\label{sec:conclusion}

This paper analyzes a nested \ac{opinf} approach for learning \acp{rom} from snapshot data and the structure of the full-order governing equations.
The nested \ac{opinf} algorithm iteratively expands the reduced space to account for the basis' intrinsic hierarchy.
It first learns \acp{rom} for the most important modes only, and then adjusts these \acp{rom} to the interactions of the less dominant modes for which the \ac{opinf} learning problem is inherently less stable.
For these later iterations, we show that the nested \ac{opinf} algorithm can increasingly choose stronger regularization without sacrificing training accuracy.
We prove that this nested strategy for regularizing the \ac{opinf} learning problem leads to at least as good an initial guess as Tikhonov regularization at the target reduced dimension, with equality only in the worst-case scenario.
Numerical tests involving a large-scale model of the Greenland demonstrated that nested \ac{opinf} can learn \acp{rom} even under strong model approximations.

The nested \ac{opinf} learning approach has versatile applications as it enables warm-starting within the \ac{opinf} training procedure:
It can be used to dynamically update a learned \ac{rom} to new snapshot data, focus on critical parts of the time or parameter domain during training, or to expand on structural model approximations.
In particular, in future work we will explore the use of nested \ac{opinf} within a greedy snapshot selection algorithm to reduce offline costs. 

\begin{appendix}

\section{Appendix}

\subsection{Regularized Operator Inference}
\label{sec:referencecode}

With Algorithm \ref{alg:opinf} we provide a brief reference code for solving the \ac{opinf} learning problem
\begin{align*}
    \min_{\widehat{\mathbf{O}} \in \R^{r \times r_{\rm{tot}}}}
    \| \Datamatrix_r \widehat{\mathbf{O}}\tr - \Rhsmatrix_r \|_F^2
    + \|\text{diag}(\mathbf{w})(\widehat{\mathbf{O}} - \widehat{\mathbf{O}}_{\rm{itr}})\tr\|_F^2
\end{align*}
with weight vector $\mathbf{w} \in \R^{r_{\rm{tot}}}$ and initial guesses $\opc_{\rm{itr}} \in \R^r, \opA_{\rm{itr}} \in \R^{r \times r}, \opH_{\rm{itr}} \in \R^{r \times r^{(2)}}$ stacked into the matrix $\widehat{\mathbf{O}}_{\rm{itr}} := [\opc_{\rm{itr}}, \opA_{\rm{itr}}, \opH_{\rm{itr}}] \in \R^{r \times r_{\rm{tot}}}$.

\begin{algorithm}
\caption{OpInf}\label{alg:opinf}
\begin{algorithmic}[1]
\smallskip
\State \textbf{Input:}
reduced dimension $r$,
data matrix $\Datamatrix_r \in \R^{K \times r_{\rm{tot}}}$, 
time derivative matrix $\Rhsmatrix_r\in \R^{K \times r}$, 
weight vector $\mathbf{w} \in \R_{>0}^{r_{\rm{tot}}}$
\State \textbf{Optional inputs:} initial guesses $\opc_{\rm{itr}} \in \R^r, \opA_{\rm{itr}} \in \R^{r \times r}, \opH_{\rm{itr}} \in \R^{r \times r^{(2)}}$ (default $\mathbf{0}$)
\State \textbf{Output:} Reduced operators $\opc_{r} \in \R^{r}$,
$\opA_{r} \in \R^{r \times r}$,
$\opH_{r} \in \R^{r \times r^{(2)}}$, smallest singular value $\sigma_{r{\rm{tot}}}$
\Statex \hrulefill
\smallskip
\Statex \texttt{\# Matrix preparation}
\State Compute residual $\mathbf{R}_{\Delta} \gets \mathbf{R}_r - \mathbf{D}_r [\opc_{r}, \opA_{r}, \opH_{r}]\tr$
\State Expand $\Datamatrix \gets [\Datamatrix_r\tr, \text{diag}(\mathbf{w})]\tr$, $\Rhsmatrix_{\Delta} \gets [\Rhsmatrix_{\Delta}\tr, \mathbf{0}_{r \times r_{\rm{tot}}}]\tr$
\smallskip
\Statex \texttt{\# Solve least squares problem}
\State Compute thin singular value decomposition $\Datamatrix = \mathbf{U} \Sigma \Psi\tr$ with $\mathbf{U} \in \R^{2K \times r_{\rm{tot}}}$ orthonormal, $\Sigma = \text{diag}(\sigma_1, \dots, \sigma_{r_{\rm{tot}}})$ ordered with $\sigma_1 \ge \dots \ge \sigma_{r_{\rm{tot}}} \ge \min(\mathbf{w}) > 0$, $\Psi \in \R^{r_{\rm{tot}} \times r_{\rm{tot}}}$
\State Compute $\widehat{\mathbf{O}}_{\Delta} \gets \Psi \Sigma^{-1} \mathbf{U}\tr \Rhsmatrix_{\Delta}$
\smallskip
\Statex \texttt{\# Operator update}
\State Split $\widehat{\mathbf{O}}_{\Delta}\tr = [\opc_{\Delta}\tr, \opA_{\Delta}\tr, \opH_{\Delta}\tr]$ into $\opc_{\Delta} \in \R^r, \opA_{\Delta} \in \R^{r \times r}, \opH_{\Delta} \in \R^{r \times s^{(r)}}$
\State Update $\opc_{r} \gets \opc_{\rm{itr}} +  \opc_{\Delta}$, 
$ \opA_{r} \gets \opA_{\rm{itr}} + \opA_{\Delta}$, 
$ \opH_{r} \gets \opH_{\rm{itr}} + \opH_{\Delta}$\\
\smallskip
\Return $\opc_{r}, \opA_{r}, \opH_{r}, \sigma_{r_{\rm{tot}}}$
\smallskip
\end{algorithmic}
\end{algorithm}

\subsection{Proofs}\label{sec:proofs}

\begin{corollary}\label{thm:opinf:conditioning}
    For any $1 \le r \le K$, let $\Datamatrix_r \in \R^{K \times r_{\rm{tot}}}$ be the data matrices for the \ac{opinf} learning problem as defined in \eqref{eq:opinf:matrices}.
    Then the minimum and maximum singular values are ordered as 
    \begin{align}
        \sigma_{\min}(\Datamatrix_{K}) \le \dots \le \sigma_{\min}(\Datamatrix_{2}) \le \sigma_{\min}(\Datamatrix_{1}) \le 
        \sigma_{\max}(\Datamatrix_{1}) \le \sigma_{\max}(\Datamatrix_{2}) \dots \le \sigma_{\max}(\Datamatrix_{K}).
    \end{align}
\end{corollary}

\begin{proof}
    For any $1 \le r \le K-1$, $\Datamatrix_r$ is a submatrix of $\Datamatrix_{r+1}$, and thus $\sigma_{\min}(\Datamatrix_{r+1}) \le \sigma_{\min}(\Datamatrix_{r}) \le 
        \sigma_{\max}(\Datamatrix_{r}) \le \sigma_{\max}(\Datamatrix_{r+1})$ (e.g., \cite{aretz2024enforcing}).
    The result follows per induction.
\end{proof}

\begin{proof}[Proof to Corollary \ref{thm:comparisonzero}]
    We abbreviate $\alpha_s(t) := \basisvec_s\tr \x(t)$, and write the \ac{fom} solution $\x(t)$ in the form $\x(t) = \sum_{s=1}^{r} \alpha_s(t) \basisvec_s + \mathbf{q}(t)$ with $\mathbf{q}(t) \perp \V_r$ uniquely defined.
    First, we apply Proposition \ref{thm:expansionerror} to get
    \begin{align*}
        \sum_{k=1}^{K} \|\V_r \widehat{\x}_r^{(0)}(t_k) - \x(t_k) \|_2^2
        &= \sum_{k=1}^{K} \|\widehat{\x}_{r-1}(t_k) - \mathbf{p}_{r-1}(t_k) \|_2^2 \\
        &\qquad + \sum_{k=1}^{K} (\alpha_r(t_k) - \alpha(0))^2 + \sum_{k=1}^{K} \|\mathbf{q}(t_k) \|^2,
    \end{align*}
    where $\widehat{\x}_{r-1} : [0, t_K] \rightarrow \R^{r-1}$ solves \eqref{eq:rom:s} with reduced dimension $s=r-1$ and the operators $\opc_{r-1} = \opc_{r-1}^{(i^*)}$, $\opA_{r-1} = \opA_{r-1}^{(i^*)}$, $\opH_{r-1} = \opH_{r-1}^{(i^*)}$ obtained at the end of iteration $s = r-1$ of Algorithm \ref{alg:flower} with $i^*$ chosen in line \ref{alg:flower:def:istar}.
    Following Remark \ref{rmk:minerrorreduction},
    \begin{align}
        \sum_{k=1}^{K} \|\widehat{\x}_{r-1}(t_k) - \mathbf{p}_{r-1}(t_k) \|_2^2
        = \delta_{r-1}^* \le \delta_{r-1}^{(0)}
        = \sum_{k=1}^{K} \|\widehat{\x}_{r-1}^{(0)}(t_k) - \mathbf{p}_{r-1}(t_k) \|_2^2
    \end{align}
    regardless of $\overline{\delta}$, and with equality if and only if $\opc_{r-1}^{(0)} = \opc_{r-1}^{(i^*)}$, $\opA_{r-1}^{(0)} = \opA_{r-1}^{(i^*)}$, $\opH_{r-1}^{(0)} = \opH_{r-1}^{(i^*)}$. 
    Applying this argument recursively yields
    \begin{align*}
        &\sum_{k=1}^{K} \|\V_r\widehat{\x}_r^{(0)}(t_k) - \x(t_k) \|_2^2 \\
        &\le \sum_{k=1}^{K} \|\widehat{\x}_{1}^{(0)}(t_k) - \mathbf{p}_1(t_k) \|_2^2 + \sum_{s=2}^{r}\sum_{k=1}^{K} (\alpha_s(t_k) + \alpha_s(0))^2 + \sum_{k=1}^{K} \|\mathbf{q}(t_k) \|^2 \\
        &= \sum_{k=1}^{K} \left(\sum_{s=1}^{r} (\alpha_s(t_k) + \alpha_s(0))^2 + \|\mathbf{q}(t_k) \|^2\right)
    \end{align*}
    where we have used for the last equality that $\opc_1 = \opA_1 = \opH_1 = \mathbf{0}_{1 \times 1}$ and thus $\widehat{\x}_1^{(0)} \equiv \alpha_1(0)$.
    For the \ac{rom} solution $\overline{\x}_r$ with the trivial operators $\opc_r = \mathbf{0}_{r \times 1}$, $\opc_r = \mathbf{0}_{r \times r}$, $\opc_r = \mathbf{0}_{r \times r^{(2)}}$ we observe that 
    \begin{align}
        \overline{\x} \equiv \V_r\tr \x(0) = [\alpha_1(0), \dots, \alpha_r(0)]\tr
    \end{align}
    and thus 
    \begin{align}
        \|\V_r \overline{\x}_r(t_k) - \x(t_k)\|_2^2 = \sum_{s=1}^{r} (\alpha_s(t_k)-\alpha_s(0))^2 + \|\mathbf{q}(t_k)\|^2.
    \end{align}
    The result follows.
\end{proof}

\end{appendix}


\section*{Declarations}

\paragraph{Acknowledgements}
The authors would like to thank Serkan Gugercin and Shane McQuarrie for fruitful discussions.
We acknowledge the World Climate Research Programme, which, through its Working Group on Coupled Modelling, coordinated and promoted the Coupled Model Intercomparison Project Phase 6 (CMIP6, \cite{eyring2016overview}). We thank the climate modeling groups for producing and making available their model output, the Earth System Grid Federation (ESGF) for archiving the data and providing access, and the multiple funding agencies who support CMIP6 and ESGF.

\paragraph{Funding}
This work was supported in parts by the Department of Energy grant DE-SC002317, and the National Science Foundation grant \#2103942.

\paragraph{Conflict of Interest}
The authors have no conflicts of interest to declare that are relevant to the content of this article.

\paragraph{Author Contributions}
N. Aretz: Conceptualization; methodology; formal analysis and investigation; writing --- original draft preparation; writing --- review and editing.
K. Willcox: Conceptualization; methodology; writing --- original draft preparation; writing --- review and editing; funding acquisition; resources; supervision.

\paragraph{Data availability}
The Greenland example in Section \ref{sec:greenland} used the following data sets:
The BedMachine v5 dataset on the Greenland bedrock topography and ice thickness is available at \href{https://nsidc.org/data/idbmg4/versions/5}{nsidc.org/data/idbmg4/versions/5}.
The GIMP ice and ocean mask (v2.0) is available at
\href{https://byrd.osu.edu/research/groups/glacier-dynamics/data/icemask}{byrd.osu.edu/research/groups/glacier-dynamics/data/icemask}.
Its combination with the coastline by Jeremie Mouginot is available at 
\url{issm.jpl.nasa.gov/documentation/tutorials/datasets/}
(SeaRISE Greenland dev1.2).
The MEaSUREs Multi-year Greenland Ice Sheet Velocity Mosaic (version 1) is available at \href{https://nsidc.org/data/nsidc-0670/versions/1}{nsidc.org/data/nsidc-0670/versions/1}.
The MEaSUREs Greenland Annual Ice Sheet Velocity Mosaics from SAR and Landsat (version 4, ID NSIDC-0725) for the years 2015-2021 are available at \href{https://nsidc.org/data/nsidc-0725/versions/4}{nsidc.org/data/nsidc-0725/versions/4}.
The CNRM-CM6-1 atmospheric projections are available on \href{https://theghub.org/}{theghub.org/}.
The Greenland heat flux A20180227-001 is available at
\href{https://ads.nipr.ac.jp/data/meta/A20180227-001/}{ads.nipr.ac.jp/data/meta/A20180227-001/}.
The Greenland heat flux ``Shapiro-Ritzwoller" is available at \href{http://ciei.colorado.edu/~nshapiro/MODEL/}{ciei.colorado.edu/$\sim$nshapiro/MODEL/}.

\paragraph{Code availability}
The code used to generate the results in this paper will be made available at \hyperlink{https://github.com/nicolearetz/Nested-OpInf}{https://github.com/nicolearetz/Nested-OpInf}.


\bibliographystyle{plain}
\bibliography{bib}

\begin{thebibliography}{10}

\bibitem{aretz2025multifidelity}
Nicole Aretz, Max Gunzburger, Mathieu Morlighem, and Karen Willcox.
\newblock Multifidelity uncertainty quantification for ice sheet simulations.
\newblock {\em Computational Geosciences}, 29(1):5, 2025.

\bibitem{aretz2024enforcing}
Nicole Aretz and Karen Willcox.
\newblock Enforcing structure in data-driven reduced modeling through nested operator inference.
\newblock In {\em 2024 63nd IEEE Conference on Decision and Control (CDC)}. IEEE, 2024.

\bibitem{azizzadenesheli2024neural}
Kamyar Azizzadenesheli, Nikola Kovachki, Zongyi Li, Miguel Liu-Schiaffini, Jean Kossaifi, and Anima Anandkumar.
\newblock Neural operators for accelerating scientific simulations and design.
\newblock {\em Nature Reviews Physics}, 6(5):320--328, 2024.

\bibitem{benner2020operator}
Peter Benner, Pawan Goyal, Boris Kramer, Benjamin Peherstorfer, and Karen Willcox.
\newblock Operator inference for non-intrusive model reduction of systems with non-polynomial nonlinear terms.
\newblock {\em Computer Methods in Applied Mechanics and Engineering}, 372:113433, 2020.

\bibitem{benner2015survey}
Peter Benner, Serkan Gugercin, and Karen Willcox.
\newblock A survey of projection-based model reduction methods for parametric dynamical systems.
\newblock {\em SIAM review}, 57(4):483--531, 2015.

\bibitem{benner2017model}
Peter Benner, Mario Ohlberger, Albert Cohen, and Karen Willcox.
\newblock {\em Model reduction and approximation: theory and algorithms}.
\newblock SIAM, 2017.

\bibitem{brunton2016discovering}
Steven~L Brunton, Joshua~L Proctor, and J~Nathan Kutz.
\newblock Discovering governing equations from data by sparse identification of nonlinear dynamical systems.
\newblock {\em Proceedings of the national academy of sciences}, 113(15):3932--3937, 2016.

\bibitem{RN3}
D.~Chen, M.~Rojas, B.H. Samset, K.~Cobb, A.~Diongue~Niang, P.~Edwards, S.~Emori, S.H. Faria, E.~Hawkins, P.~Hope, P.~Huybrechts, M.~Meinshausen, S.K. Mustafa, G.-K. Plattner, and A.-M. Tréguier.
\newblock {\em Framing, Context, and Methods}, page 147–286.
\newblock Cambridge University Press, Cambridge, United Kingdom and New York, NY, USA, 2021.

\bibitem{eyring2016overview}
Veronika Eyring, Sandrine Bony, Gerald~A Meehl, Catherine~A Senior, Bjorn Stevens, Ronald~J Stouffer, and Karl~E Taylor.
\newblock Overview of the coupled model intercomparison project phase 6 (cmip6) experimental design and organization.
\newblock {\em Geoscientific Model Development}, 9(5):1937--1958, 2016.

\bibitem{filanova2023mechanical}
Yevgeniya Filanova, Igor~Pontes Duff, Pawan Goyal, and Peter Benner.
\newblock An operator inference oriented approach for linear mechanical systems.
\newblock {\em Mechanical Systems and Signal Processing}, 200:110620, 2023.

\bibitem{geelen2024learning}
Rudy Geelen, Laura Balzano, Stephen Wright, and Karen Willcox.
\newblock Learning physics-based reduced-order models from data using nonlinear manifolds.
\newblock {\em Chaos: An Interdisciplinary Journal of Nonlinear Science}, 34(3), 2024.

\bibitem{geng2025data}
Yuwei Geng, Lili Ju, Boris Kramer, and Zhu Wang.
\newblock Data-driven reduced-order models for port-{H}amiltonian systems with operator inference.
\newblock {\em Computer Methods in Applied Mechanics and Engineering}, 442:118042, 2025.

\bibitem{geng2024gradient}
Yuwei Geng, Jasdeep Singh, Lili Ju, Boris Kramer, and Zhu Wang.
\newblock Gradient preserving {O}perator {I}nference: {D}ata-driven reduced-order models for equations with gradient structure.
\newblock {\em Computer Methods in Applied Mechanics and Engineering}, 427:117033, 2024.

\bibitem{ghattas2021acta}
Omar Ghattas and Karen Willcox.
\newblock Learning physics-based models from data: {P}erspectives from inverse problems and model reduction.
\newblock {\em Acta Numerica}, 30:445–554, 2021.

\bibitem{gkimisis2025spatiallylocal}
Leonidas Gkimisis, Nicole Aretz, Marco Tezzele, Thomas Richter, Peter Benner, and Karen~E. Willcox.
\newblock Non-intrusive reduced-order modeling for dynamical systems with spatially localized features.
\newblock {\em Computer Methods in Applied Mechanics and Engineering}, 444:118115, 2025.

\bibitem{goyal2023guaranteed}
Pawan Goyal, Igor~Pontes Duff, and Peter Benner.
\newblock Guaranteed stable quadratic models and their applications in sindy and operator inference.
\newblock {\em arXiv preprint arXiv:2308.13819}, 2023.

\bibitem{gruber2023hamiltonian}
Anthony Gruber and Irina Tezaur.
\newblock Canonical and noncanonical {H}amiltonian operator inference.
\newblock {\em Computer Methods in Applied Mechanics and Engineering}, 416, 2023.

\bibitem{gruber2025variational}
Anthony Gruber and Irina Tezaur.
\newblock Variationally consistent {H}amiltonian model reduction.
\newblock {\em SIAM Journal on Applied Dynamical Systems}, 24(1):376--414, 2025.

\bibitem{hesthaven2016certified}
Jan~S Hesthaven, Gianluigi Rozza, Benjamin Stamm, et~al.
\newblock {\em Certified reduced basis methods for parametrized partial differential equations}, volume 590.
\newblock Springer, 2016.

\bibitem{RN1}
IPCC.
\newblock {\em Climate Change 2021: The Physical Science Basis. Contribution of Working Group I to the Sixth Assessment Report of the {I}ntergovernmental {P}anel on {C}limate {C}hange}, volume In Press.
\newblock Cambridge University Press, Cambridge, United Kingdom and New York, NY, USA, 2021.

\bibitem{koike2024energy}
Tomoki Koike and Elizabeth Qian.
\newblock Energy-preserving reduced operator inference for efficient design and control.
\newblock In {\em AIAA SciTech 2024 Forum}, page 1012, 2024.

\bibitem{kramer2024survey}
Boris Kramer, Benjamin Peherstorfer, and Karen Willcox.
\newblock Learning nonlinear reduced models from data with operator inference.
\newblock {\em Annual Review of Fluid Mechanics}, 56:521–548, 2024.

\bibitem{issm}
Eric Larour, Helene Seroussi, Mathieu Morlighem, and Eric Rignot.
\newblock Continental scale, high order, high spatial resolution, ice sheet modeling using the {I}ce {S}heet {S}ystem {M}odel ({ISSM}).
\newblock {\em Journal of Geophysical Research: Earth Surface}, 117(F1), 2012.

\bibitem{MacAyeal1989}
D.~R. MacAyeal.
\newblock Large-scale ice flow over a viscous basal sediment: {T}heory and application to {I}ce {S}tream {B}, {A}ntarctica.
\newblock {\em J. Geophys. Res.}, 94(B4):4071--4087, APR 10 1989.

\bibitem{mcquarrie2021data}
Shane~A McQuarrie, Cheng Huang, and Karen~E Willcox.
\newblock Data-driven reduced-order models via regularised operator inference for a single-injector combustion process.
\newblock {\em Journal of the Royal Society of New Zealand}, 51(2):194--211, 2021.

\bibitem{McQuarrie2021c}
Shane~A McQuarrie, Parisa Khodabakhshi, and Karen~E Willcox.
\newblock Nonintrusive reduced-order models for parametric partial differential equations via data-driven operator inference.
\newblock {\em SIAM Journal on Scientific Computing}, 45(4):A1917--A1946, 2023.

\bibitem{mcquarrie2023data}
Shane~Alexander McQuarrie.
\newblock {\em Data-driven parametric reduced-order models: Operator inference for reactive flow applications}.
\newblock PhD thesis, 2023.

\bibitem{BedMachineV5}
M.~et~al Morlighem.
\newblock {I}ce{B}ridge {B}ed{M}achine {G}reenland, version 5, 2022.

\bibitem{morlighem2017bedmachine}
Mathieu Morlighem, Chris~N Williams, Eric Rignot, Lu~An, Jan~Erik Arndt, Jonathan~L Bamber, Ginny Catania, Nolwenn Chauch{\'e}, Julian~A Dowdeswell, Boris Dorschel, et~al.
\newblock Bedmachine v3: Complete bed topography and ocean bathymetry mapping of {G}reenland from multibeam echo sounding combined with mass conservation.
\newblock {\em Geophysical research letters}, 44(21):11--051, 2017.

\bibitem{nowicki2020experimental}
Sophie Nowicki, Antony~J Payne, Heiko Goelzer, Helene Seroussi, William~H Lipscomb, Ayako Abe-Ouchi, C{\'e}cile Agosta, Patrick Alexander, Xylar~S Asay-Davis, Alice Barthel, et~al.
\newblock Experimental protocol for sealevel projections from {ISMIP6} standalone ice sheet models.
\newblock {\em The Cryosphere Discussions}, 2020:1--40, 2020.

\bibitem{Peherstorfer2020a}
Benjamin Peherstorfer.
\newblock {Sampling low-dimensional Markovian dynamics for preasymptotically recovering reduced models from data with operator inference}.
\newblock {\em SIAM Journal on Scientific Computing}, 42(5):A3489--A3515, 2020.

\bibitem{peherstorfer2022breaking}
Benjamin Peherstorfer.
\newblock Breaking the kolmogorov barrier with nonlinear model reduction.
\newblock {\em Notices of the American Mathematical Society}, 69(5):725--733, 2022.

\bibitem{Peherstorfer2016c}
Benjamin Peherstorfer and Karen Willcox.
\newblock {Data-driven operator inference for nonintrusive projection-based model reduction}.
\newblock {\em Computer Methods in Applied Mechanics and Engineering}, 306:196--215, 2016.

\bibitem{qian2022reduced}
Elizabeth Qian, Ionut-Gabriel Farcas, and Karen Willcox.
\newblock Reduced operator inference for nonlinear partial differential equations.
\newblock {\em SIAM Journal on Scientific Computing}, 44(4):A1934--A1959, 2022.

\bibitem{Qian2020}
Elizabeth Qian, Boris Kramer, Benjamin Peherstorfer, and Karen Willcox.
\newblock {Lift {\&} {L}earn: {P}hysics-informed machine learning for large-scale nonlinear dynamical systems}.
\newblock {\em Physica D: Nonlinear Phenomena}, 406:132401, 2020.

\bibitem{Sawant2021}
Nihar Sawant, Boris Kramer, and Benjamin Peherstorfer.
\newblock Physics-informed regularization and structure preservation for learning stable reduced models from data with operator inference.
\newblock {\em Computer Methods in Applied Mechanics and Engineering}, 404:115836, 2023.

\bibitem{schmid2022dynamic}
Peter~J Schmid.
\newblock Dynamic mode decomposition and its variants.
\newblock {\em Annual Review of Fluid Mechanics}, 54(1):225--254, 2022.

\bibitem{sharma2024preserving}
Harsh Sharma and Boris Kramer.
\newblock Preserving {L}agrangian structure in data-driven reduced-order modeling of large-scale mechanical systems.
\newblock {\em Physica D: Nonlinear Phenomena}, 462:134128, 2024.

\bibitem{sharma2024lagrangian}
Harsh Sharma, David~A Najera-Flores, Michael~D Todd, and Boris Kramer.
\newblock Lagrangian operator inference enhanced with structure-preserving machine learning for nonintrusive model reduction of mechanical systems.
\newblock {\em Computer Methods in Applied Mechanics and Engineering}, 423:116865, 2024.

\bibitem{sharma2022hamiltonian}
Harsh Sharma, Zhu Wang, and Boris Kramer.
\newblock Hamiltonian operator inference: {P}hysics-preserving learning of reduced-order models for canonical {H}amiltonian systems.
\newblock {\em Physica D: Nonlinear Phenomena}, 431:133122, 2022.

\bibitem{uy2023rollouts}
Wayne Isaac~Tan Uy, Dirk Hartmann, and Benjamin Peherstorfer.
\newblock Operator inference with roll outs for learning reduced models from scarce and low-quality data.
\newblock {\em Computers \& Mathematics with Applications}, 145:224–239, 2023.

\bibitem{Uy2021}
Wayne Isaac~Tan Uy and Benjamin Peherstorfer.
\newblock {Operator inference of non-Markovian terms for learning reduced models from partially observed state trajectories}.
\newblock {\em Journal of Scientific Computing}, 88(3):1--31, 2021.

\bibitem{voldoire2018cnrm}
Aurore Voldoire.
\newblock Cnrm-cerfacs cnrm-cm6-1 model output prepared for cmip6 cmip.
\newblock 2018.

\end{thebibliography}

\end{document}